\documentclass{article}

\PassOptionsToPackage{numbers, compress, sort}{natbib}

\usepackage[main, final]{neurips_2026}
\usepackage[utf8]{inputenc} % allow utf-8 input
\usepackage[T1]{fontenc}    % use 8-bit T1 fonts

\usepackage{hyperref}

\usepackage{url}            % simple URL typesetting
\usepackage{booktabs}       % professional-quality tables
\usepackage{amsfonts}       % blackboard math symbols
\usepackage{nicefrac}       % compact symbols for 1/2, etc.
\usepackage{microtype}      % microtypography
\usepackage{xcolor}         % colors
\usepackage{graphicx}

\usepackage{amsmath}
\usepackage{amssymb}
\usepackage{mathtools}
\usepackage{amsthm}

\usepackage{algorithm}
\usepackage{algpseudocode}

\usepackage[capitalize,noabbrev]{cleveref}

\usepackage{thmtools}
\usepackage{thm-restate}

\theoremstyle{plain}
\newtheorem{theorem}{Theorem}[section]

\theoremstyle{definition}

\theoremstyle{remark}
\newtheorem{remark}[theorem]{Remark}

\usepackage[textsize=tiny]{todonotes}

\usepackage{accents}
\newcommand{\ubar}[1]{\underaccent{\bar}{#1}}

\usepackage{wrapfig}
\usepackage{dsfont}
\usepackage{bm}
\usepackage{etoolbox}

\newcommand{\mbf}[1]{\mathbf{#1}}

\newcommand{\zb}{\mbf{z}}

\newcommand{\bb}{\mbf{b}}

\newcommand{\ubhat}{\hat{\mbf{u}}}
\newcommand{\lbhat}{\hat{\mbf{l}}}
\newcommand{\xb}{\mbf{x}}

\newcommand{\thetab}{\bm{\theta}}

\DeclareMathOperator*{\argmax}{\arg\!\max}

\newcommand{\wcloss}{\mathcal{L}^{\mathcal{C}_{\epsilon}} (\xb, y)}
\newcommand{\advloss}{\ubar{\mathcal{L}}^{\mathcal{C}_{\epsilon}} (\xb, y)}
\newcommand{\verloss}{\bar{\mathcal{L}}^{\mathcal{C}_{\epsilon}} (\xb, y)}
\newcommand{\zdifflb}{\ubar{\zb}^{\mathcal{C}_{\epsilon}}(\xb, y)}
\newcommand{\sabrzdifflb}[1]{\ubar{\zb}^{\mathcal{C}_{\epsilon}^{#1}}(\xb, y)}
\newcommand{\zdiffcc}{\zb_{\text{cc}}^{\mathcal{C}_{\epsilon}}(\alpha; \xb, y)}
\newcommand{\zdiffopt}{\zb^{\mathcal{C}_{\epsilon}}(\xb, y)}
\newcommand{\ccibp}{\mathcal{L}_{\text{cc}}^{\mathcal{C}_{\epsilon}} (\alpha; \xb, y)}
\newcommand{\exprl}[1]{\mathcal{L}^{\mathcal{C}_{\epsilon}} (#1; \xb, y)}
\newcommand{\latentdifflb}{\ubar{h}^{\mathcal{C}_{\epsilon}}(\xb)}
\newcommand{\latentdiffub}{\bar{h}^{\mathcal{C}_{\epsilon}}(\xb)}
\newcommand{\latentdiffcclb}{\ubar{h}_{\text{cc}}^{\mathcal{C}_{\epsilon}}(\alpha; \xb)}
\newcommand{\latentdiffccub}{\bar{h}_{\text{cc}}^{\mathcal{C}_{\epsilon}}(\alpha; \xb)}
\newcommand{\tlatent}{h^t (\xb)}
\newcommand{\distillationloss}{\mathcal{R}^{\mathcal{C}_{\epsilon}} (\xb, y)}
\newcommand{\distillationlosscc}[1]{\mathcal{R}_{\text{cc}}^{\mathcal{C}_{\epsilon}} (#1; \xb, y) }
\newcommand{\distillationlossub}{\bar{\mathcal{R}}^{\mathcal{C}_{\epsilon}} (\xb, y) }
\newcommand{\distillationlosssabr}{\bar{\mathcal{R}}^{\mathcal{C}_{\epsilon}^\alpha} (\xb, y) }
\newcommand{\distillationlosslb}{\ubar{\mathcal{R}}^{\mathcal{C}_{\epsilon}} (\xb, y) }

\usepackage{enumitem}

\usepackage{stmaryrd}

\usepackage[labelformat=simple]{subcaption}

\usepackage{adjustbox}
\usepackage{multirow}
\renewcommand{\bfseries}{\fontseries{b}\selectfont}
\newrobustcmd{\B}{\bfseries}
\newrobustcmd{\IT}{\itshape}
\usepackage[normalem]{ulem}
\newrobustcmd{\U}{\uline}
\def\Decimal{.000}% This structure should corresponds to the ".2" of "table-format"

\def\Ulinehelp#1.#2 {%
	#1.#2\setbox0=\hbox{#1\Decimal}\hspace{-\wd0}{\if\relax#2\relax%
		\uline{\phantom{#1.0}}\else\uline{\phantom{#1.#2}}\fi}%
}
\usepackage{siunitx}

\title{Learning Better Certified Models from \\Empirically-Robust Teachers}

\author{%
  Alessandro De Palma\thanks{Work partly carried out at Inria, École Normale Supérieure, PSL University, CNRS.} \\
  Department of Statistics\\
  London School of Economics and Political Science\\
  London, United Kingdom \\
  \texttt{a.de-palma@lse.ac.uk} \\
}

\begin{document}

\maketitle

\begin{abstract}
  \looseness=-1	
  Adversarial training attains strong empirical robustness to specific adversarial attacks by training on concrete adversarial perturbations, but it produces neural networks that are not amenable to strong robustness certificates through neural network verification. 
  On the other hand, earlier certified training schemes directly train on bounds from network relaxations to obtain models that are certifiably robust, but display sub-par standard performance.
  Recent work has shown that state-of-the-art trade-offs between certified robustness and standard performance can be obtained through a family of losses combining adversarial outputs and neural network bounds.
  Nevertheless, differently from empirical robustness, verifiability still comes at a significant cost in standard performance.
  In this work, we propose to leverage empirically-robust teachers to improve the performance of certifiably-robust models through knowledge distillation. 
  Using versatile feature-space distillation objectives, we show that distillation from adversarially-trained teachers can significantly improve the performance of state-of-the-art certified training algorithms for ReLU networks on robust computer vision benchmarks.
\end{abstract}

\section{Introduction}

\looseness=-1
Adversarial examples~\citep{Biggio2013,Szegedy2014,Goodfellow2015}, imperceptible input perturbations that induce unintended behaviors such as misclassifications, hinder the applicability of deep learning systems to safety-critical applications.
Formal robustness certificates can be obtained through neural network verification algorithms~\citep{Tjeng2019,Lomuscio2017,Ehlers2017}.
However, these techniques have a worst-case runtime that is exponential in network size even on piecewise-linear models~\citep{Katz2017}.
While \emph{empirical} robustness to specific adversarial attacks can be attained by training against adversarial inputs~\citep{Madry2018}, a technique known as adversarial training, neural network verifiers fail to provide robustness certificates on the resulting networks in a reasonable time.
This is in spite of sustained progress in verification algorithms, which now couple specialized network convex relaxations~\citep{Zhang2018,xu2021fast,Wong2018,sparsealgosDePalma2024} with efficient divide and conquer strategies~\citep{Bunel2020,HenriksenLomuscio21} within hardware-accelerated branch-and-bound frameworks~\citep{betacrown,Ferrari2022,improvedbab}.

\looseness=-1
Networks amenable to robustness certificates through neural network verification can be obtained using so-called certified training schemes.
Earlier approaches~\citep{Wong2018,Zhang2020,Gowal2019,Mirman2018} proposed to train networks by computing the loss on neural network bounds obtained from convex relaxations, effectively deploying a building block of network verifiers within the training loop.
Counter-intuitively, and owing to their relative smoothness and continuity, the loosest relaxations were found to outperform the others in this context~\citep{Jovanovic2022,Lee2021}.
While the resulting networks enjoy strong verifiability using the same relaxations employed at training time, this is achieved at a significant cost in standard performance.
Relying on branch-and-bound frameworks to perform post-training verification, a more recent line of works has shown that better trade-offs between certified robustness and standard performance can be obtained by combining methods based on convex relaxations with adversarial training~\citep{Balunovic2020,DePalma2022,Mueller2023,Mao2023}.
This was then formalized into the notion of \emph{expressivity}~\citep{DePalma2024}, which entails the ability of a certified training loss to span a continuous range of trade-offs between pure adversarial training and the earlier losses based on convex relaxations, and can be easily implemented through convex combinations.

\looseness=-1
While, as attested by a recent study~\citep{mao2024ctbench}, expressive losses produce state-of-the-art certifiably-robust networks for the $\ell_\infty$ threat model, their standard performance is still far from ideal.
%For instance, the best-reported standard accuracy on downscaled ImageNet for certifiably-robust models is $22.73\%$~\citep{DePalma2024}, whereas non-robust models can easily reach well above $60\%$~\citep{Chrabaszcz2017}.
Noting that adversarially-trained models display significantly better standard performance and empirical robustness than certifiably-robust models~\citep{croce2021robustbench}, we believe they could be directly employed to improve the certified training process.
Specifically, we aim to produce better certifiably-robust models by performing knowledge distillation~\citep{hinton2015distilling,romero2015fitnets} from an empirically-robust teacher to the target model, leading to the following contributions:
\begin{itemize}[leftmargin=1.3em]
	\item \looseness=-1 Inspired by a worst-case loss common in the adversarial robustness literature, we introduce a novel and versatile feature-space distillation framework for certified training ($\S$\ref{sec:distillation}).
	\item \looseness=-1 Focusing on ReLU networks and on the $\ell_\infty$ threat model, we present distillation losses tailored to a series of popular certified training algorithms ($\S$\ref{sec:ccdist}, $\S$\ref{sec:dist-other}), demonstrating that distillation from adversarially-trained teachers can significantly improve the robustness-accuracy trade-offs of state-of-the-art certified training algorithms on robust computer vision benchmarks ($\S$\ref{sec:main-exp}). 
	\item \looseness=-1 We show that students successfully learn from an adversarially-trained teacher while at the same time significantly surpassing it in terms of certified robustness ($\S$\ref{sec:teacher-student-exp}). In other words, they partly inherit the benefits of the representations from empirically-robust models while preserving verifiability through certified training losses.
\end{itemize}

\looseness=-1
We believe our work to be a further step towards bridging the ever-present gap between certified and empirical adversarial robustness.
Code and sample trained models are available at \url{https://github.com/alessandrodepalma/feature-distillation-for-certified-training}.

\section{Background} \label{sec:background}

\looseness=-1
Let lowercase letters denote scalars (e.g., $a \in \mathbb{R}$) boldface lowercase letters denote vectors (e.g., $\mbf{a} \in \mathbb{R}^n$), uppercase letters denote matrices (e.g., $A \in \mathbb{R}^{n \times m}$), calligraphic letters denote sets (e.g. $\mathcal{A} \subset \mathbb{R}^n$) and brackets denote intervals (e.g., $[\mbf{a}, \mbf{b}]$).
Furthermore, let lowercase single-letter subscripts denote vector indices (for instance, $\mbf{a}_i$, or $f(\mbf{a})_i$ for a vector-valued function $f$), and let the vector of all entries of $\mbf{a}$ except its $i$-th entry be denoted by $\mbf{a}_{\bar{i}}$.
We will write $\mbf{1}^n \in \mathbb{R}^n$ for the unit vector, and $I^{n} \in \mathbb{R}^{n \times n}$ for the identity matrix.
Abusing notation, given a vector-valued objective function $a(\xb)$ we will denote by $\min_{\xb \in \mathcal{A}} a(\xb)$ the vector storing the minimum of $\min_{\xb \in \mathcal{A}} a(\xb)_i$ in its $i$-th entry.
Finally, we use the following shorthands: $[A]_+ := \max\{0, A\}$, $[A]_- := \min\{0, A\}$, $\llbracket a, b \rrbracket := \{a, a + 1, \dots, b\}$.
Appendix~\ref{sec:app-notation} provides further details and clarifications on the employed notation.

\looseness=-1
Let $(\xb, y) \sim \mathcal{D}$ be a $k$-way classification dataset with points $\xb \in \mathbb{R}^{d}$ and labels $y \in \llbracket 1, k \rrbracket $. 
We aim to train a feed-forward ReLU neural network $f_{\thetab} : \mathbb{R}^{d} \rightarrow \mathbb{R}^{k}$ with parameters $\thetab \in \mathbb{R}^p$ such that each point $\xb$ from $\mathcal{D}$ and the allowed adversarial perturbations $\mathcal{C}_{\epsilon}(\xb) :=  \left\{ \xb' : \left\lVert \xb' - \xb  \right\rVert_{\infty} \leq \epsilon \right\}$ around it are correctly classified:
\begin{equation}
	\xb' \in \mathcal{C}_{\epsilon}(\xb) \implies \argmax_{i} f_{\thetab}(\xb)_i = y.
	\label{eq:robust-loss}
\end{equation}
For ease of notation, we will henceforth omit network parameters $\thetab$.

\subsection{Neural network verification} \label{sec:nn-verification}

\looseness=-1
Neural network verification can be used to check whether equation~\eqref{eq:robust-loss} holds on a given network $f$, providing a deterministic robustness certificate.
Let $\zb_{f}(\xb, y)$ denote the differences between the ground truth logits, i.e., unnormalized network outputs corresponding to the correct class, and the other logits: $\zb_{f}(\xb, y) := \left(f(\xb)_y \mbf{1}^k - f(\xb) \right)$. 
Verifiers solve the following optimization problem:
\begin{equation}
	\zdiffopt := \min_{\xb' \in \mathcal{C}_{\epsilon}(\xb)}  \zb_{f}(\xb', y).
	\label{eq:verification}
\end{equation}
If all entries of $\zdiffopt_{\bar{y}}$ are positive, then equation~\eqref{eq:robust-loss} holds, implying that the network is guaranteed to be robust to any given attack in $\mathcal{C}_{\epsilon}(\xb)$.
An algorithm computing $\zdiffopt$ exactly is called a complete verifier: this is typically done through branch-and-bound~\citep{Bunel2018,sparsealgosDePalma2024,betacrown,Ferrari2022,HenriksenLomuscio21}.
As this was shown to be NP-complete~\citep{Katz2017}, a series of algorithms propose to compute less expensive lower bounds $\zdifflb \leq \zdiffopt$~\citep{Zhang2018,Wong2018,Dvijotham2018,Singh2019} by operating on network relaxations (incomplete verifiers), with $\zdifflb_{\bar{y}} > 0$ successfully providing a robustness certificate.
\looseness=-1
The least expensive incomplete verifier is called Interval Bound Propagation (IBP)~\citep{Gowal2019,Mirman2018}, which is obtained by applying interval arithmetics~\citep{Sunaga1958,Hickey2001} to the network operators. We refer the reader to appendix~\ref{sec:ibp} for further details.

\looseness=-1
Beyond neural network verifiers, the robustness of a network can be certified using Lipschitz bounds~\citep{fazlyab2019efficient}, and through randomized smoothing~\citep{Cohen2019}, the latter providing a probabilistic certificate. Both techniques are at their most effective on the $\ell_2$ threat model, which is beyond the scope~of~this~work.

\subsection{Certified training} \label{sec:certified-training}

\looseness=-1
In principle, a network can be trained for verified robustness (certified training) by replacing the employed classification loss $\mathcal{L} : \mathbb{R}^k \times \llbracket 1, k \rrbracket \rightarrow \mathbb{R}$ (e.g., cross-entropy) with its worst-case across the adversarial perturbations~\citep{Madry2018}:
\begin{equation}
	\wcloss := \max_{\xb' \in\ \mathcal{C}_{\epsilon}(\xb)} \mathcal{L} (f(\xb'), y).
	\label{eq:worst-case}
\end{equation}

\paragraph{Bounding the worst-case loss}
However, computing $\wcloss$ exactly is as hard as solving problem~\eqref{eq:verification}, and is hence typically replaced by an approximation.
Let $\xb_{\text{adv}}$ denote the output of an adversarial attack, for instance PGD~\citep{Madry2018}, on the network $f$.
Lower bounds to $\wcloss$ can be obtained by evaluating the loss at $\xb_{\text{adv}}$:
\begin{equation}
	\advloss := \mathcal{L} (f(\xb_{\text{adv}}), y) \leq  \wcloss.
	\label{eq:adversarial-loss}
\end{equation}
\looseness=-1
Networks trained using $\mathcal{L} (f(\xb_{\text{adv}}), y)$ (adversarial training) typically enjoy strong empirical robustness to the same types of attacks employed during training, but are not amenable to formal guarantees, which are the focus of this work.
Assume the loss is monotonically increasing with respect to the non-ground-truth network logits and that it is translation-invariant: $\mathcal{L} (-\zb_{f}(\xb, y), y) = \mathcal{L} (f(\xb), y)$~\citep{Wong2018}. This holds for common losses such as cross-entropy.
Given bounds $\zdifflb$ from an incomplete verifier, let us set $\zdifflb_y=0$ noting that, by definition, $\zb_{f}(\xb, y)_y=0$.
An upper bound to $\wcloss$ can then be obtained as:
\begin{equation}
	 \verloss := \mathcal{L} (-\zdifflb, y) \geq \wcloss.
	\label{eq:verified-loss}
\end{equation}
While networks trained via $\mathcal{L} (-\zdifflb, y)$ display strong verifiability through the same incomplete verifiers used to compute the $\zdifflb$ bounds, the most successful option being IBP~\citep{Jovanovic2022}, they display sub-par standard performance.
%
%Certifiably robust models can be alternatively obtained by building networks with low Lipschitz constants~\citep{hu2023unlocking,leino2021globally,prach2024,Meunier22,Zhang2022rethinking}: these techniques are orthogonal to our work. 

\paragraph{Expressive losses}
\looseness=-1
More recent and successful methods rely on losses featuring a combination between adversarial attacks and bounds from incomplete verifiers~\citep{DePalma2022,Balunovic2020,Mao2023,Mueller2023}, pairing better standard performance with stronger verifiability through branch-and-bound.
In particular, the ability of a loss function to span a continuous range of trade-offs between lower and upper bounds to $\wcloss$, termed \emph{expressivity}, was found to be crucial to maximize performance~\citep{DePalma2024}.
A parametrized loss function $\exprl{\alpha}$ is termed~\citep{DePalma2024} expressive if: 
(i) $\exprl{0} = \advloss$ and $\exprl{1} = \verloss$;
(ii) $\advloss \leq \exprl{\alpha} \leq \verloss$ $\forall\ \alpha \in [0, 1]$;
(iii)  $\exprl{\alpha}$ is continuous and monotonically increasing for $\alpha \in [0, 1]$.

Nevertheless, the best-performing certifiably-robust models still display worse robustness-accuracy trade-offs than empirically-robust models~\citep{DePalma2024,croce2021robustbench}.

While we focus on certified training via convex relaxations in this work, we refer the reader to $\S$\ref{sec:related} for an overview of complementary techniques.

\subsection{Knowledge distillation} \label{sec:dist-background}

\looseness=-1
Knowledge distillation~\citep{hinton2015distilling} was introduced as a way to transfer the predictive ability of a large teacher model $t : \mathbb{R}^d \rightarrow \mathbb{R}^k$ onto the target model $f$, termed the student.
Let $\text{KL}_{T} : \mathbb{R}^k \times \mathbb{R}^k \rightarrow \mathbb{R}$ denote a KL divergence incorporating a softmax operation with temperature $T$.
In its standard form, knowledge distillation pairs the employed classification loss with a KL divergence term where the teacher logits, termed soft labels, act as target distribution:
\begin{equation}
	\begin{aligned}
			\mathcal{L}^{\text{KD}}(\lambda; \xb, y)  := & \enskip \mathcal{L} (f(\xb), y)  +  T^2 \lambda\ \text{KL}_{T} (f(\xb), t(\xb)).
	\end{aligned}
	\label{eq:std-distillation}
\end{equation}
In order to provide more granular teacher information to the student, a series of works perform knowledge distillation on the intermediate activations~\citep{romero2015fitnets,zagoruyko2017paying,heo2019comprehensive}: this is referred to as feature-space distillation.
Let us write both the teacher and the student as the composition of a classification head with a feature map: $t := g^t \circ h^t$ and $f := g \circ h$, respectively.
In its simplest form, when the feature spaces of the teacher and student share the same dimensionality $w$, feature-space distillation encourages similarity between teacher and student features through a squared $\ell_2$ term:
\begin{equation}
	\label{eq:clean-fs-distillation}
	\mathcal{L}^{\text{F-KD}}(\lambda; \xb, y)  := \mathcal{L} (f(\xb), y) + \lambda \left\lVert   h (\xb) - \tlatent \right\rVert^2_2.
\end{equation}
Both $\mathcal{L}^{\text{KD}}(\lambda; \xb, y)$ and $\mathcal{L}^{\text{F-KD}}(\lambda; \xb, y) $ were originally designed to transfer standard network performance from teacher to student, and do not take robustness into account.
Recent works have therefore focused on designing specialized distillation schemes that transfer either empirical adversarial robustness~\citep{goldblum2020adversarially,zhu2022reliable,zi2021revisiting,muhammad2021mixacm} or probabilistic certified robustness~\citep{vaishnavi2022feasibility} from teacher to student. 
Aiming to improve state-of-the-art trade-offs between deterministic certified robustness and standard performance, we here propose to transfer knowledge from an empirically-robust teacher through feature-space distillation.

\section{Distillation for certified robustness} \label{sec:method}

\looseness=-1
Motivated by their superior standard performance and empirical robustness, we aim to leverage adversarially-trained networks to train better certifiably-robust models. 
We will first present our distillation framework in $\S$\ref{sec:distillation}, then show how they can be used in conjunction with existing certified training algorithms ($\S$\ref{sec:ccdist}, $\S$\ref{sec:dist-other}).
Pseudo-code and proofs of technical results are respectively provided in appendices~\ref{sec:pseudo-code}~and~\ref{sec:app-proofs}.

\subsection{Distillation from empirically-robust teachers} \label{sec:distillation}

As outlined in $\S$\ref{sec:certified-training}, adversarial and certified training employ different approximations of the ideal robust loss in equation~\eqref{eq:worst-case}.
While certified training schemes attain strong verifiability, they do so at the expense of both empirical robustness and standard performance compared to adversarial training, which arguably learns a more empirically-effective representation of data.
We propose to employ feature-space distillation to transfer some of the benefits of this representation, while adopting state-of-the-art certified training losses to enforce verifiability.

\paragraph{Worst-case distillation loss}
\looseness=-1
Certified training is concerned with worst-case network behavior.
Mirroring the robust loss in $\S$\ref{sec:certified-training}, a worst-case feature-space distillation loss over the perturbations would take the following form:
\begin{equation}
	\distillationloss := \max_{
		\xb' \in \mathcal{C}_{\epsilon}(\xb)} \left\lVert   h (\xb') - \tlatent \right\rVert^2_2.
	\label{eq:ideal-distillation-loss}
\end{equation}
However, as for $\wcloss$, computing $\distillationloss$ exactly amounts to solving a non-convex optimization problem over the features. 
Similarly to what done for the robust loss, we propose to use approximations of equation~\eqref{eq:ideal-distillation-loss}. These will be tailored to the employed certified training loss.

\paragraph{Bounding the distillation loss}
\looseness=-1
As for equation~\eqref{eq:adversarial-loss}, a lower bound $\distillationlosslb$ can be obtained by evaluating the left-hand side of equation~\eqref{eq:ideal-distillation-loss} at the adversarial input $\xb_{\text{adv}}$ computed on the student~model:
\begin{equation*}
	\distillationlosslb := \left\lVert   h (\xb_{\text{adv}}) - \tlatent \right\rVert^2_2 \leq \distillationloss.
\end{equation*}
\looseness=-1
Similarly to $\verloss$, an upper bound $\distillationlossub$ can be instead computed resorting to IBP.
\begin{restatable}{proposition}{ubloss}
	Let $\latentdifflb$ and $\latentdiffub$ respectively denote IBP lower and upper bounds to the student features:
	\begin{equation*}
		\latentdifflb \leq h (\xb') \leq \latentdiffub \text{, } \quad \forall \xb' \in \mathcal{C}_{\epsilon}(\xb).
	\end{equation*}
	The loss function \[\distillationlossub := \sum_i \max \left\{
	\left(\latentdifflb_i - \tlatent_i \right)^2 , \left(\latentdiffub_i - \tlatent_i \right)^2\right\}\] is an upper bound to the worst-case distillation loss from equation~\eqref{eq:ideal-distillation-loss}:
	$\distillationlossub \geq \distillationloss$.
	\label{prop:ubloss}
\end{restatable} 
\looseness=-1
Inspired by the parallels between the worst-case robust loss and distillation loss, and their approximations, we aim to use $\distillationlosslb$ and $\distillationlossub$ as building blocks for algorithm-specific distillation losses introduced in the next subsections.

\paragraph{Teacher models}
\looseness=-1
We propose to use teacher models trained via pure adversarial training (see $\S$\ref{sec:certified-training}) on the target task, employing the same architecture as the student.
Appendix~\ref{sec:app-clean-teachers} empirically demonstrates that the empirical robustness of the teacher is essential for the success of the distillation process.
Owing to its empirical efficacy (see appendix~\ref{sec:app-earlier-latents}) we define the features as the activations before the last affine network layer.

\subsection{CC-Dist}  \label{sec:ccdist}

Aiming to showcase the efficacy of our distillation framework, we primarily focus on distilling onto CC-IBP, a simple state-of-the-art expressive loss. 
The CC-IBP loss is obtained by evaluating $\mathcal{L}$ on convex combinations $\zdiffcc := ( 1 - \alpha)\ \zb_{f}(\xb_{\text{adv}}, y) + \alpha\ \zdifflb$ between adversarial logit differences $\zb_{f}(\xb_{\text{adv}}, y)$ and their lower bounds $\zdifflb$ computed via IBP: \[\ccibp := \mathcal{L}\left( -\zdiffcc, y\right).\]
\looseness=-1
Mirroring the properties of the CC-IBP loss, we aim to design a parametrized feature-space distillation loss $\distillationlosscc{\alpha}$ that can span a continuous range of trade-offs between $\distillationlosslb$ and $\distillationlossub$.
As we next show, $\distillationlosscc{\alpha}$ can be realized by distilling onto the following convex combinations of the adversarial student features $h (\xb_{\text{adv}})$ with $\latentdiffub$ and $\latentdifflb$:
\begin{equation}
	\begin{aligned}
		& \latentdiffccub := (1 - \alpha)\ h (\xb_{\text{adv}}) + \alpha\ \latentdiffub, \\
		& \latentdiffcclb := (1 - \alpha)\ h (\xb_{\text{adv}}) + \alpha\ \latentdifflb.%
	\end{aligned}% 
	\label{eq:latentdiffcc}
\end{equation}
\looseness=-1
\begin{restatable}{proposition}{ccloss}
	The loss function 
	\begin{equation*}
		\distillationlosscc{\alpha} :=  \sum_i \max \left\{
		\left(
		\latentdiffcclb_i -  \tlatent_i
		\right)^2, 
		\left(
		\latentdiffccub_i -  \tlatent_i
		\right)^2
		\right)
	\end{equation*}
	enjoys the following properties:
	%	\begin{enumerate}
		(i) $\distillationlosscc{0} = \distillationlosslb$ and $\distillationlosscc{1} = \distillationlossub$;
		(ii) $\distillationlosslb \leq \distillationlosscc{\alpha} \leq \distillationlossub$ $\forall \ \alpha \in [0, 1]$; 
		(iii) $\distillationlosscc{\alpha}$ is continuous and monotonically increasing for $\alpha \in [0, 1]$.
		%	\end{enumerate}
	\label{prop:ccloss}
\end{restatable} 

\looseness=-1
While we expect special cases $\distillationlosscc{0}$ and $\distillationlosscc{1}$ to work well on settings where CC-IBP is employed with small or large $\alpha$ values, respectively, we aim to leverage the full flexibility from proposition~\ref{prop:ccloss} to obtain consistent performance across setups.
Intuitively, $\distillationlosscc{\alpha}$ will chiefly operate on the adversarial student features when the robust loss is close to $\mathcal{L} (f(\xb_{\text{adv}}), y)$ (low $\alpha$), and mainly focus on the student IBP bounds when robust loss is close to $\mathcal{L} (-\zdifflb, y)$ (large $\alpha$).

\looseness=-1
As we next show, if the student uses an affine classification head, the CC-IBP convex combinations $\zdiffcc$ 
correspond to a simple affine function of $\latentdiffcclb$ and $\latentdiffccub$, the feature-space convex combinations onto which distillation is performed.
As stated in $\S$\ref{sec:distillation}, this assumption holds throughout our experiments, and is found to be empirically more effective through an ablation (appendix~\ref{sec:app-earlier-latents}).
\begin{restatable}{lemma}{convexcombination}
	\label{prop:convexcombination}
	\looseness=-1
	Let the student classification head $g(\mbf{a})$ be affine, and let  $\tilde{g}^y(\mbf{a}) := \tilde{W}^{n} \mbf{a} + \tilde{\mbf{b}}^{n}$ be its composition with the operator performing the difference between logits.
	We can write the CC-IBP convex combinations as:%
	\begin{equation*}
		\zdiffcc = [\tilde{W}^n]_+ \latentdiffcclb + [\tilde{W}^n]_- \latentdiffccub + \tilde{\bb}^n.
	\end{equation*}%
\end{restatable}%
Lemma~\ref{prop:convexcombination} provides a strong coupling between $\distillationlosscc{\alpha}$ and the CC-IBP loss when their convex combinations use the same $\alpha$ parameter. We hence incorporate this design choice into CC-Dist (short for CC-IBP Distillation).
Omitting any regularization and introducing the distillation coefficient $\beta$, the CC-Dist training loss takes the following form:
\begin{equation}
	\begin{aligned}
		\mathcal{L}^{\text{CC-Dist}} (\alpha, \beta; \xb, y) := \ccibp + \beta\ \distillationlosscc{\alpha}.
	\end{aligned}
	\label{eq:loss}
\end{equation}
\looseness=-1
In practice, denoting by $w$ the dimensionality of the feature space, unless stated otherwise, we employ $\beta= \nicefrac{5}{w}$ in all reported experiments owing to its consistent performance.% across setups.

\subsection{Distilling onto other certified training schemes} \label{sec:dist-other}

Our distillation framework can be, in principle, applied on top of any existing certified training scheme for ReLU networks.
We here propose distillation losses tailored to SABR~\citep{Mueller2023}, MTL-IBP~\citep{DePalma2024}, and IBP-R~\citep{DePalma2022}. 

\paragraph{SABR-Dist}
\looseness=-1
SABR~\citep{Mueller2023} proposes to compute IBP bounds over a parametrised subset $\mathcal{C}_{\epsilon}^\alpha$ of $\mathcal{C}_{\epsilon}$ that includes an adversarial attack $\xb_{\text{adv}}$, termed a ``small box''. We denote the resulting bounds by $\sabrzdifflb{\alpha}$. 
The SABR loss $\mathcal{L}\left( -\sabrzdifflb{\alpha}, y\right)$ is expressive owing to the monotonic growth of $-\sabrzdifflb{\alpha}$ with $\alpha \in [0, 1]$ from the adversarial logit differences to $-\zdifflb$~\citep[proposition B.2]{DePalma2024}.
Hence, in order to mirror the versatility of $\distillationlosscc{\alpha}$, we propose to employ a distillation term that computes $\distillationlossub$ on 
$\mathcal{C}_{\epsilon}^\alpha$, calling the resulting algorithm SABR-Dist:
\[\mathcal{L}^{\text{SABR-Dist}} (\alpha, \beta; \xb, y) := \mathcal{L}\left( -\sabrzdifflb{\alpha}, y\right) + \beta\ \distillationlosssabr.\]
Similarly to $\distillationlosscc{\alpha}$, $\distillationlosssabr$ will mainly operate on the student adversarial features when $\mathcal{C}_{\epsilon}^\alpha$ is a small neighborhood of the attack (low $\alpha$), while focusing on the student IBP bounds for larger $\alpha$.
At the same time, differently from $\distillationlosscc{\alpha}$, $\distillationlosssabr$ can be computed re-using the same IBP bounds employed for the SABR loss.

\paragraph{MTL-Dist}
\looseness=-1
MTL-IBP~\citep{DePalma2024} is an expressive loss obtained by taking convex combinations between $\advloss$ and $\verloss$: $(1- \alpha) \advloss + \alpha \verloss$. Owing to its similarity with CC-IBP, whose loss it upper bounds for convex losses and $\alpha \in (0, 1)$~\citep[proposition 4.2]{DePalma2024}, we advocate for the use of $\distillationlosscc{\alpha}$ as for CC-Dist, calling the resulting algorithm MTL-Dist: \[\mathcal{L}^{\text{MTL-Dist}} (\alpha, \beta; \xb, y) := (1- \alpha) \advloss + \alpha \verloss + \beta \distillationlosscc{\alpha}.\]

\paragraph{IBP-R-Dist}
\looseness=-1
IBP-R proposes to add verifiability on top of adversarial training through a regularizer $\mathcal{A}^{\mathcal{C}_{\epsilon}} (\xb, y)$ on the area of network relaxations~\citep{DePalma2022} (see appendix~\ref{sec:ibp} for a definition).
Owing to its strong reliance on the adversarial loss, we propose to perform distillation directly on the adversarial features. Specifically, we use the adversarial lower bound to the worst-case distillation loss, referring to the resulting algorithm as IBP-R-Dist: \[\mathcal{L}^{\text{IBP-R-Dist}} (\lambda, \beta; \xb, y) := \advloss + \lambda \mathcal{A}^{\mathcal{C}_{\epsilon}} (\xb, y) + \beta \distillationlosslb.\]
Differently from the other methods that we propose, IBP-R-Dist does not operate on an expressive loss, showcasing the versatility of the proposed distillation framework.
\section{Related work} \label{sec:related}

Owing to its favorable optimization properties~\citep{Jovanovic2022}, the relatively loose IBP ($\S$\ref{sec:nn-verification}) is the method of choice~\citep{mao2024understanding} when computing lower bounds to equation~\eqref{eq:verification}  to be employed for certified training.
Tighter techniques are however preferred within branch-and-bound-based complete verifiers.
Branch-and-bound~\citep{Bunel2018} couples a bounding algorithm with a strategy to refine the network relaxations by iteratively splitting problem~\eqref{eq:verification} into subproblems (branching), which typically operates by splitting piecewise linearities into their linear components~\citep{improvedbab,HenriksenLomuscio21,Ferrari2022,Bunel2020}.
Earlier bounding algorithms relied on linear network relaxations~\citep{Bunel2018,Ehlers2017,Anderson2020}, with more effective techniques operating in the dual space~\citep{Dvijotham2018,BunelDP20,sparsealgosDePalma2024}.
State-of-the-art branch-and-bound frameworks have since converged to using fast bounds based on propagating couples of linear bounds through the network~\citep{Zhang2018,Singh2019,xu2021fast} with Lagrangian relaxations employed to capture additional constraints~\citep{betacrown,Ferrari2022,zhang2022gcpcrown,zhou2024scalable}.

\looseness=-1
Earlier certified training works add geometric regularizers to the standard adversarial training loss~\citep{Xiao2019,Croce2019provable,Liu2021training}. 
Recent certified training schemes mix adversarial training with bounds from network relaxations to obtain strong post-training verifiability using branch-and-bound (see $\S$\ref{sec:certified-training}).
\citet{Balunovic2020} perform adversarial attacks over latent-space over-approximations. 
The good performance of the latter at lower perturbation radii was matched by IBP-R, and later refined by connecting the gradients of the over-approximation with those of the attack (TAPS)~\citep{Mao2023}.
The later SABR (see $\S$\ref{sec:dist-other}) obtained significant performance improvements, then generalized into the notion of expressive losses~\citep{DePalma2024} (see $\S$\ref{sec:certified-training}), which includes losses based on convex combinations forming the basis of this work.
All the above methods train standard feedforward ReLU networks using convex relaxations, with IBP bounds being the most prevalent among them. Recent work has highlighted the superior representational power of networks trained via tighter relaxations~\citep{Baader2024,Mao2024}, and sought to overcome the associated optimization challenges~\citep{Balauca2025}.

Certifiably-robust networks can also be obtained by training 1-Lipschitz models using alternative architectures~\citep{Zhang2021towards,Zhang2022boosting}, the best-performing being SortNet~\citep{Zhang2022rethinking}.
These were also shown to benefit from additional generated data~\citep{Altstidl2024}, whose use is beyond the focus of our work.
While we here concentrate on $\ell_\infty$ perturbations and deterministic certificates, 1-Lipschitz architectures~\citep{prach2024,Meunier22} and Lipschitz regularization~\citep{hu2023unlocking,leino2021globally} are very effective when $\mathcal{C}_{\epsilon}$ is defined using the $\ell_2$ norm, setting under which probabilistic robustness certificates can be effectively obtained using randomised smoothing~\citep{Cohen2019}.

\looseness=-1
Related work on adversarial distillation is provided in appendix~\ref{sec:related-distillation}.

\section{Experimental evaluation} \label{sec:experiments}

\looseness=-1
This section presents an experimental evaluation of the proposed distillation framework on image classification datasets from the certified training literature.
We first present results on the effectiveness of distillation from adversarially-trained teachers for various certified training algorithms ($\S$\ref{sec:main-exp}), followed by a comparison of CC-Dist models with their teachers ($\S$\ref{sec:teacher-student-exp}), and by an analysis of the effect of the distillation coefficient on CC-Dist ($\S$\ref{sec:sensitivity-study}).

\looseness=-1
In line with previous work~\citep{Mueller2023,Mao2023,Shi2021,DePalma2024}, all the models in the main evaluation are based on a $7$-layer convolutional architecture named \texttt{CNN-7}, regularized using an $\ell_1$ term.
The teacher models are trained via adversarial training based on a 10-step PGD adversary~\citep{Madry2018} and $\ell_1$ regularization.
As this was found to be beneficial in practice, we employ the \texttt{CNN-7} architecture for the teacher model too. 
Similarly to the expressive losses codebase~\citep{DePalma2024}, our implementation is in PyTorch~\citep{Paszke2019} and relies on \texttt{auto\_LiRPA}~\citep{Xu2020} to compute IBP bounds.
As common in the wider certified training literature~\citep{Mueller2023,Mao2023}, we employ the custom regularization and initialization introduced by~\citet{Shi2021}.
Post-training verification is performed using the OVAL branch-and-bound framework~\citep{Bunel2018,BunelDP20,improvedbab} similarly to~\citet{DePalma2022,DePalma2024}, using a configuration based on $\alpha$-$\beta$-CROWN network bounds~\citep{betacrown,xu2021fast} and employing at most $600$ seconds per image.
Additional details and supplementary experiments can be found in appendices~\ref{sec:app-experiments-details}~and~\ref{sec:app-supp-experiments}, respectively. 

\begin{table*}[b!]
	%	\vspace{-7pt}		
	\sisetup{detect-all = true}
	
	\setlength{\heavyrulewidth}{0.09em}
	\setlength{\lightrulewidth}{0.5em}
	\setlength{\cmidrulewidth}{0.03em}
	
	\centering
	
	\scriptsize
	\setlength{\tabcolsep}{4pt}
	\aboverulesep = 0.3mm  % 0.605mm
	\belowrulesep = 0.1mm  % 0.984mm
	\caption{%\small 
		\looseness=-1
		Evaluation of the effect of CC-Dist compared to pure CC-IBP~\citep{DePalma2024} when training for certified robustness against $\ell_\infty$ norm perturbations.
		Means over $6$ repetitions and their $95\%$ Confidence Intervals (CIs) are reported: improvements with non-overlapping CIs are highlighted in bold.
		\label{fig:ccdist-results}}
	\begin{adjustbox}{width=.85\textwidth, center}
		\begin{tabular}{
				c
				c
				l
				S[separate-uncertainty,table-format=2.2(2)]
				S[separate-uncertainty,table-format=2.2(2)]
				S[separate-uncertainty,table-format=2.2(2)]
			} 
			
			\toprule \\[-5pt]
			
			\multicolumn{1}{ c }{Dataset} &
			\multicolumn{1}{ c }{$\epsilon$} &
			\multicolumn{1}{ c }{Method} &
			\multicolumn{1}{ c }{Standard acc. [\%]} &
			\multicolumn{1}{ c }{AA acc. [\%]} &
			\multicolumn{1}{ c }{Certified acc. [\%]}  \\
			
			\cmidrule(lr){1-2} \cmidrule(lr){2-2} \cmidrule(lr){3-6} \\[-5pt]
			
			\multirow{4}{*}{CIFAR-10} & \multirow{2}{*}{$\frac{2}{255}$} &
			
			\multicolumn{1}{ c }{\textsc{CC-Dist}} 
			& \B 81.59(15) & \B 70.62(17) & \B 64.43(24) \\
			
			& & \multicolumn{1}{ c }{\textsc{CC-IBP}} 
			& 79.83(29) & 69.03(31) & 63.51(25)  \\
			
			\cmidrule(lr){2-6} \\[-6pt]
			
			& \multirow{2}{*}{$\frac{8}{255}$} &
			
			\multicolumn{1}{ c }{\textsc{CC-Dist}} 
			& \B 55.18(25) & 36.40(35) &  35.05(50) \\
			
			& & \multicolumn{1}{ c }{\textsc{CC-IBP}} 
			& 54.13(32) & 36.06(42) & 34.89(39)  \\
			
			\cmidrule{1-6} \\[-6pt]
			
			\multirow{2}{*}{TinyImageNet} & \multirow{2}{*}{$\frac{1}{255}$} &
			
			\multicolumn{1}{ c }{\textsc{CC-Dist}} 
			& \B 43.86(21)  & \B 32.30(7) & \B 27.61(21) \\
			
			& & \multicolumn{1}{ c }{\textsc{CC-IBP}} 
			& 41.31(40) & 30.00(25) & 26.45(24)  \\
			
			\cmidrule{1-6} \\[-6pt]
			
			\multirow{2}{*}{ImageNet64} & \multirow{2}{*}{$\frac{1}{255}$} &
			
			\multicolumn{1}{ c }{\textsc{CC-Dist}} 
			& \B 28.62(94) & \B 18.35(69) & \B 14.12(19) \\ 
			
			& & \multicolumn{1}{ c }{\textsc{CC-IBP}} 
			& 26.41(95)  & 16.75(52) & 13.56(15) \\
			
			\bottomrule 
		\end{tabular}
	\end{adjustbox}
	%	\vspace{-7pt}
\end{table*}

\subsection{Distilling onto certified training} \label{sec:main-exp}

\looseness=-1
We here evaluate our distillation framework by assessing its effect on top of state-of-the-art certified training schemes, demonstrating that distillation from empirically-robust teachers can significantly enhance their trade-offs between certified robustness and standard performance.
We first focus on CC-Dist, then demonstrate the more general efficacy of distillation through a set of SABR-Dist, MTL-Dist and IBP-R-Dist experiments.
We rely on the popular AutoAttack~\citep{Croce2020} to measure empirical robustness in this context.

\paragraph{CC-Dist} 

\looseness=-1
Table~\ref{fig:ccdist-results} shows that the distillation loss from $\S$\ref{sec:ccdist} results in improved average robustness-accuracy trade-offs compared to CC-IBP, a state-of-the-art certified training algorithm, on all the considered CIFAR-10~\citep{CIFAR10}, TinyImageNet~\citep{Le2015}, and downscaled ImageNet ($64\times64$)~\citep{Chrabaszcz2017} benchmarks.
On CIFAR-10 with $\epsilon=\nicefrac{2}{255}$, and on TinyImageNet and downscaled ImageNet, CC-Dist yields relatively large and statistically-significant improvements on all metrics (see appendix~\ref{sec:significance} for a comparison with previously-attained improvements on these benchmarks).
In particular, CC-Dist improves standard accuracy by more than $2\%$ on both TinyImageNet and downscaled ImageNet while at the same time providing non-negligible average certified accuracy improvements.
For the larger \linebreak $\epsilon=\nicefrac{8}{255}$ on CIFAR-10, distillation results in a statistically-significant improvement on standard accuracy while leaving robustness metrics roughly unvaried, demonstrating a strictly-better performance trade-off.
Appendix~\ref{sec:exp-sota} compares CC-Dist with literature results in the area, demonstrating that the distillation improvements lead to a new state-of-the-art for certified training on ReLU networks.

\begin{table}[b!]
	%	\vspace{-7pt}
	\sisetup{detect-all = true}
	
	\centering
	
	\scriptsize
	\setlength{\tabcolsep}{4pt}
	\aboverulesep = 0.3mm  % 0.605mm
	\belowrulesep = 0.1mm  % 0.984mm
	\caption{\small 
		\looseness=-1
		Effect of distillation from empirically-robust teachers onto SABR~\citep{Mueller2023}, MTL-IBP~\citep{DePalma2024}, and IBP-R~\citep{DePalma2022}. We report means over $4$ repetitions and their $95\%$ CIs, with improvements having non-overlapping CIs highlighted in bold. 
		\label{fig:sabr-distillation}}
	\vspace{4pt}
	\begin{adjustbox}{width=.9\textwidth, center}
		
		\begin{tabular}{
				c
				c
				l
				S[separate-uncertainty,table-format=2.2(2)]
				S[separate-uncertainty,table-format=2.2(2)]
				S[separate-uncertainty,table-format=2.2(2)]
			} 
			
			\toprule \\[-5pt]
			
			\multicolumn{1}{c}{Dataset} &
			\multicolumn{1}{c}{$\epsilon$} &
			\multicolumn{1}{c}{Method} &
			\multicolumn{1}{c}{Standard acc. [\%]}  &
			\multicolumn{1}{c}{AA acc. [\%]} &
			\multicolumn{1}{c}{Certified acc. [\%]}  \\
			
			\cmidrule(lr){1-2} \cmidrule(lr){3-3} \cmidrule(lr){4-6} \\[-5pt]
			
			\multirow{11}{*}{CIFAR-10} & \multirow{6.5}{*}{$\frac{2}{255}$} &
			
			\multicolumn{1}{ c }{\textsc{SABR-Dist}} 
			& \B 81.22(19)	& \B  70.09(39)  & 64.33(61)\\
			
			& & \multicolumn{1}{ c }{\textsc{SABR}}  
			& 79.38(28) & 68.25(38) & 63.54(46)  \\ 
			
			\cmidrule(lr){3-6} \\[-6pt]
			
			& & \multicolumn{1}{ c }{\textsc{MTL-Dist}} 
			& \B  81.69(29)	& \B 70.51(50)  & \B 64.09(28) \\ 
			
			& & \multicolumn{1}{ c }{\textsc{MTL-IBP}} 
			& 79.87(31) & 68.71(33) & 63.21(13)  \\ 
			
			\cmidrule(lr){3-6} \\[-6pt]
			
			& & \multicolumn{1}{ c }{\textsc{IBP-R-Dist}}
			& \B 80.96(19)	& \B 68.98(49)  & 61.74(29) \\
			
			& & \multicolumn{1}{ c }{\textsc{IBP-R}} 
			& 79.55(20) & 67.53(62) & 60.84(70)  \\ 
			
			\cmidrule(lr){2-6} \\[-6pt]
			
			& \multirow{4.5}{*}{$\frac{8}{255}$} &
			
			\multicolumn{1}{ c }{\textsc{SABR-Dist}} 
			& 53.14(69)	& 36.16(55) & 35.11(54) \\
			
			& & \multicolumn{1}{ c }{\textsc{SABR}}  
			& 52.34(18) & 36.03(30) & 35.16(29) \\ 
			
			\cmidrule(lr){3-6} \\[-5pt]
			
			& & \multicolumn{1}{ c }{\textsc{MTL-Dist}} 
			& 54.86(48) &36.16(86) & 34.95(70) \\ 
			
			& & \multicolumn{1}{ c }{\textsc{MTL-IBP}} 
			& 54.22(55) & 36.09(32) & 35.06(47)  \\ 
			
			\cmidrule(lr){1-6} \\[-6pt]
			
			\multirow{6.5}{*}{TinyImageNet} & \multirow{6.5}{*}{$\frac{1}{255}$} &
			
			\multicolumn{1}{ c }{\textsc{SABR-Dist}} 
			& \B 43.92(25)	& \B  32.32(40)  & 27.31(36) \\
			
			& & \multicolumn{1}{ c }{\textsc{SABR}}  
			& 41.13(37) & 29.77(85) & 26.46(85)  \\ 
			
			\cmidrule(lr){3-6} \\[-6pt]
			
			& & \multicolumn{1}{ c }{\textsc{MTL-Dist}} 
			& \B  43.85(38)	& \B 32.23(28)  & \B 27.26(20) \\ 
			
			& & \multicolumn{1}{ c }{\textsc{MTL-IBP}} 
			& 41.20(39) & 29.86(39) &  26.37(18) \\ 
			
			\cmidrule(lr){3-6} \\[-6pt]
			
			& & \multicolumn{1}{ c }{\textsc{IBP-R-Dist}}
			& \B 44.19(18)	& \B 32.15(38)  & 27.02(59) \\
			
			& & \multicolumn{1}{ c }{\textsc{IBP-R}} 
			& 41.10(76) & 29.59(98) & 25.95(66)  \\ 
			
			\bottomrule 
			
		\end{tabular}
	\end{adjustbox}
\end{table}

\paragraph{SABR-Dist, MTL-Dist and IBP-R-Dist} 
\looseness=-1
Recalling that ($\S$\ref{sec:dist-other}) knowledge distillation from empirically-robust teachers can be applied to a wide range of certified training algorithms, we aim to show its efficacy beyond CC-Dist.
In particular, we here show results pertaining to the distillation losses tailored to SABR, MTL-IBP and IBP-R (see $\S$\ref{sec:dist-other}).
We forego downscaled ImageNet experiments owing to their large computational cost, and IBP-R distillation on CIFAR-10 with $\epsilon=\nicefrac{8}{255}$ because of its significantly worse performance trade-offs on this benchmark (see table~\ref{fig:ccdist-literature}).
Table~\ref{fig:sabr-distillation} shows that the distillation process significantly improves the robustness-accuracy trade-offs of all the three considered methods on CIFAR-10 with $\epsilon=\nicefrac{2}{255}$ and TinyImageNet, with statistically-significant improvements on standard accuracy and empirical robustness on all three methods. 
In both settings, MTL-Dist also displays statistically-significant certified accuracy improvements, owing to smaller experimental variability.
Differently from CC-Dist, distillation does not appear to provide significant improvements for neither MTL-IBP nor SABR on CIFAR-10 with $\epsilon=\nicefrac{8}{255}$. Given the well-know superiority of Lipschitz-based techniques on this benchmark (see table~\ref{fig:ccdist-literature}), we believe that ReLU networks may have fundamental limitations in this regime.
Exploring distillation on Lipschitz networks is an exciting direction for future work.
Overall, these results demonstrate the wider potential of our framework, and confirm that distillation can improve on the state-of-the-art for certified training through convex relaxations.

\paragraph{Supplementary results}
\looseness=-1
Appendix~\ref{sec:distillation-special-cases} investigates the effectiveness of special cases of the $\distillationlosscc{\alpha}$ distillation loss on CIFAR-10, showing that $\distillationlosscc{0}$ and standard clean distillation from equation~\eqref{eq:clean-fs-distillation}, which requires an extra clean pass, perform competitively with $\distillationlosscc{\alpha}$ for $\epsilon=\nicefrac{2}{255}$.
Appendix~\ref{sec:app-diff-architectures} shows that CC-Dist is effective on a different model architecture and activation function.
These further point to the benefits of knowledge distillation from adversarially-robust teachers for certified training.
Appendices~\ref{sec:logit-distillation},~\ref{sec:app-earlier-latents},~and~\ref{sec:app-t-adv-logits} respectively investigate the performance of logit-based distillation, of using earlier layer features for $\distillationlosscc{\alpha}$, and of distilling from adversarial teacher features.

\subsection{Understanding distillation: teacher-student comparison} \label{sec:teacher-student-exp}

\begin{table*}[t!]
	%	\vspace{-7pt}
	\sisetup{detect-all = true}
	
	\centering
	
	\scriptsize
	\setlength{\tabcolsep}{4pt}
	\aboverulesep = 0.3mm  % 0.605mm
	\belowrulesep = 0.1mm  % 0.984mm
	\caption{\small 
		Teacher-student comparison for CC-Dist models (first repetition from table~\ref{fig:ccdist-results}).
		\label{fig:teacher-student}}
	\begin{adjustbox}{width=.87\textwidth, center}
		
		\begin{tabular}{
				c
				c
				l
				S[table-format=2.2]
				S[table-format=2.2]
				S[table-format=2.2]
			} 
			
			\toprule \\[-5pt]
			
			\multicolumn{1}{c}{\multirow{1}{*}{Dataset}} &
			\multicolumn{1}{c}{\multirow{1}{*}{$\epsilon$}} &
			\multicolumn{1}{c}{\multirow{1}{*}{Model}} &
			\multicolumn{1}{c}{Standard acc. [\%]}  &
			\multicolumn{1}{c}{AA acc. [\%]} &
			\multicolumn{1}{c}{Certified acc.$^\dagger$ [\%]} \\
			
			\cmidrule(lr){1-2} \cmidrule(lr){3-3} \cmidrule(lr){4-6} \\[-5pt]
			
			\multirow{5}{*}{CIFAR-10} & \multirow{2}{*}{$\frac{2}{255}$} &
			
			\multicolumn{1}{ c }{\textsc{Student}}
			& 81.55 & 70.71 &  64.4 \\
			
			& & \multicolumn{1}{ c }{\textsc{Teacher}} 
			& 88.23  & 73.26 &  1.2 \\ 
			
			\cmidrule(lr){2-6} \\[-5pt]
			
			& \multirow{2}{*}{$\frac{8}{255}$} &
			
			\multicolumn{1}{ c }{\textsc{Student}} 
			& 55.13  & 36.80 &  35.8 \\
			
			& & \multicolumn{1}{ c }{\textsc{Teacher}} 
			& 78.16  & 41.62 &  0.0 \\
			
			\cmidrule{1-6} \\[-5pt]
			
			\multirow{2}{*}{TinyImageNet} & \multirow{2}{*}{$\frac{1}{255}$} &
			
			\multicolumn{1}{ c }{\textsc{Student}} 
			&  43.78   & 32.30 &  32.6 \\
			
			& & \multicolumn{1}{ c }{\textsc{Teacher}}
			&  47.18  & 34.17 &  17.4 \\
			
			\cmidrule{1-6} \\[-5pt]
			
			\multirow{2}{*}{ImageNet64} & \multirow{2}{*}{$\frac{1}{255}$} &
			
			\multicolumn{1}{ c }{\textsc{Student}} 
			&  28.17  & 17.87  &  5.6 \\ 
			
			& & \multicolumn{1}{ c }{\textsc{Teacher}}
			&  40.30  & 25.34 &  0.0 \\  
			
			\bottomrule 
			
			\\[-5pt]
			\multicolumn{6}{ l }{
				\shortstack[l]{$^\dagger$Certified accuracy reported over the first $500$ test images.} 
			}
			
		\end{tabular}
	\end{adjustbox}
	%	\vspace{-3pt}
	%	\vspace{-5pt}
\end{table*}
\looseness=-1
In order to provide insights on the distillation process, we now compare CC-Dist models from table~\ref{fig:ccdist-results} (their first repetition) with their respective adversarially-trained teacher models.
Owing to the large computational cost associated to running branch-and-bound on networks trained via pure adversarial training (the teachers), we restrict the certified robustness comparison on the first $500$ images of the test set of each benchmark.
Table~\ref{fig:teacher-student} shows a clear difference in terms of standard accuracy and empirical robustness to AutoAttack~\citep{Croce2020} between teacher and student, the latter displaying worse performance as expected.
On the other hand, the teachers fail to attain good certified robustness across any of the considered benchmarks, highlighting that the students significantly surpass the respective teachers in terms of verifiability.

\looseness=-1
In other words, the distillation process successfully exploits empirically-robust teachers to improve the performance of certifiably-robust models by (i) transferring some of their superior natural accuracy and empirical robustness, and (ii) preserving a large degree of verifiability via the employed certified training loss.

Appendix~\ref{sec:app-clean-teachers} shows that employing clean teachers results in worse performance profiles, highlighting the importance of robust representations.

\begin{figure*}[t!]
	\hspace{15pt}
	\begin{subfigure}{0.40\textwidth}
		\centering
		\includegraphics[width=\textwidth]{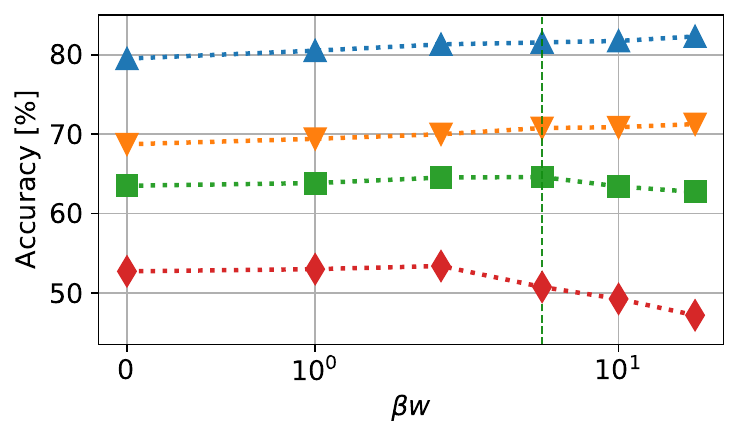}
		\vspace{-15pt}
		\captionsetup{justification=centering, labelsep=space, margin=6pt}
		\caption{\label{fig:beta-sensitivity-2-255}  $\epsilon=\nicefrac{2}{255}$.}
	\end{subfigure}\hfill
	\begin{subfigure}{0.40\textwidth}
		\centering
		\includegraphics[width=\columnwidth]{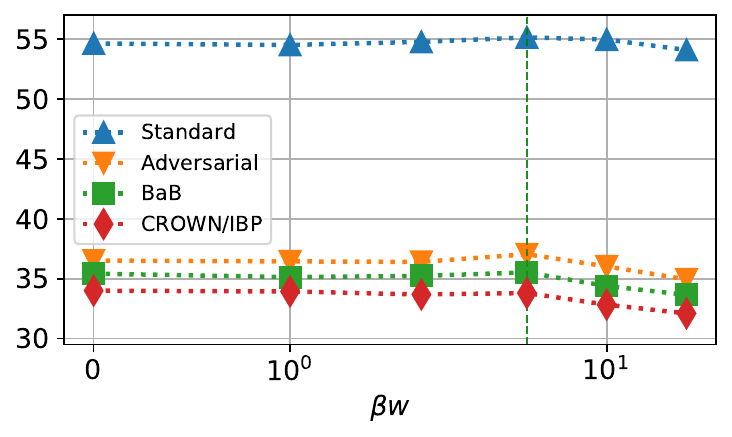}
		\vspace{-15pt}
		\captionsetup{justification=centering, labelsep=space, margin=6pt}
		\caption{\label{fig:alpha-sensitivity-8-255}  $\epsilon=\nicefrac{8}{255}$.}
	\end{subfigure}\hspace{15pt}
	\caption{\label{fig:beta-sensitivity} 
		\looseness=-1
		Standard, empirical adversarial and certified accuracies (BaB and CROWN/IBP) under $\ell_\infty$ perturbations of networks trained using CC-Dist under varying distillation coefficient $\beta$. 
		The legend is reported once for all subfigures in plot~\ref{fig:alpha-sensitivity-8-255}.
		Metrics are reported on the CIFAR-10 test set.
		The $\beta$ value employed throughout the paper (see $\S$\ref{sec:ccdist}) is marked by a dashed vertical line. 
	}
	%	\vspace{-10pt}
\end{figure*}

\subsection{Sensitivity to distillation coefficient} \label{sec:sensitivity-study}

\looseness=-1
Owing to its strong performance across the considered settings and to avoid the introduction of an additional hyper-parameter, $\beta = \nicefrac{5}{w}$ throughout all experiments.
We now present an empirical study of its effect on CC-Dist performance.

Figure~\ref{fig:beta-sensitivity} compares the behavior at $\beta = \nicefrac{5}{w}$ (the first repetition from table~\ref{fig:ccdist-results}), where both certified robustness via branch-and-bound and standard performance are enhanced, with results for different distillation coefficients. 
It reports standard performance, empirical adversarial robustness (measured through the attacks within the employed branch-and-bound framework), and certified robustness (both using branch-and-bound and the best results between incomplete verifiers CROWN~\citep{Zhang2018} and IBP~\citep{Gowal2019}) across $\beta$ values on the CIFAR-10 test set.
For $\epsilon=\nicefrac{2}{255}$, increasing $\beta$ leads to a steady increase in both standard accuracy and empirical robustness. 
At the same time, verifiability increases until an intermediate $\beta$ value, then steadily decreases, with the certified accuracy via incomplete verifiers peaking at a lower value than when using branch-and-bound.
This is similar to the behavior of the $\alpha$ parameter in expressive losses such as CC-IBP~\citep{DePalma2024}.
As also visible in table~\ref{fig:ccdist-results}, distillation appears to be less helpful for $\epsilon=\nicefrac{8}{255}$, where standard performance is enhanced only for intermediate $\beta$ values.
We explain this difference through the larger degree of similarity between teacher and student on $\epsilon=\nicefrac{2}{255}$: see table~\ref{fig:teacher-student}.

\section{Conclusions}

\looseness=-1
We presented a novel distillation framework that successfully leverages the superior standard accuracy and empirical robustness of adversarially-trained models to train better certifiably-robust neural networks.
Mirroring approximations to a worst-case loss popular in the adversarial machine learning literature, we present a series of novel and versatile feature-space distillation losses tailored to popular certified training algorithms.
We show that these can significantly improve the robustness-accuracy trade-offs of state-of-the-art certified training schemes on benchmarks from the robust vision literature, in spite of the relatively low certified robustness of the employed teacher models.
This points to the large potential of knowledge distillation for certified robustness.
%
%As a result, CC-Dist attains a new state-of-the-art in certified training across ReLU architectures, showcasing the potential of knowledge distillation for certified robustness.
%
%Results in appendix demonstrate that distillation from empirically-robust teachers is beneficial for other certified training schemes and on a different activation function, pointing to its broader applicability.
%
Indeed, while our work focuses on teachers that employ the same architecture as the student, we are hopeful that further progress can be made by appropriately training teachers with larger effective capacity.

%\clearpage

\begin{ack}
	The work was supported by the ``SAIF" project, funded by the ``France 2030'' government investment plan managed by the French National Research Agency, under the reference ANR-23-PEIA-0006. % SAIF
	The authors have no competing interests to declare.
    We thank the anonymous reviewers for providing valuable and constructive feedback.
\end{ack}

\bibliography{ccibp-distillation}
\bibliographystyle{abbrvnat}

\clearpage
\appendix
%\onecolumn
\section{Limitations and broader impact} \label{sec:app-limitations-broader}

% limitations
While our work demonstrates that knowledge distillation from empirically-robust teachers can be employed to improve the state-of-the-art in certified training, it still shares many of the limitations of work in the area.
We here focus on higher-level arguments, pointing the reader to appendix~\ref{sec:hyperparameters} for lower-level experimental details.
Our work focuses on $\ell_\infty$ perturbations, the context in which certified training via convex relaxations on \texttt{CNN-7} is the state-of-the-art deterministic certified defense (except on larger perturbations, were Lipschitz-based methods display better performance: see appendix~\ref{sec:exp-sota}).
A large certified accuracy to local $\ell_\infty$ perturbations of a fixed radius does not generally guarantee neither robustness to other types of perturbations, nor to the same perturbations when applied to a point subject to a strong distribution shift.
Devising effective distillation techniques for other perturbation types, network architectures, and training methodologies are exciting avenues for future work.
Furthermore, while our work targets verifiability via branch-and-bound, we acknowledge the large associated runtime cost, which prevents such verifiers to be deployed on a per-point basis in a real-time environment.
Finally, the considered benchmarks all belong to the vision domain: effectively deploying certified training schemes to other data domains, for instance natural language processing, will involve overcoming additional domain-specific difficulties.

% broader impact
\looseness=-1
We do not anticipate any short-term negative impact of certifiably-robust networks.
Indeed, we believe that efficient certified training techniques should be primarily leveraged towards providing guarantees to deep learning systems operating in safety-critical contexts. 
Nevertheless, we acknowledge that adversarial attacks may have social utility when deployed against unethical systems, which constitute unintended use of machine learning technologies, pointing to a potential shortcoming of provable~robustness.
Effective mitigation strategies may include targeted regulations.

\section{Related work on knowledge distillation} \label{sec:related-distillation}

Many specialized knowledge distillation schemes designed to transfer empirical robustness focus on modifying the logit-space distillation loss from equation~\eqref{eq:std-distillation}.
Assuming an adversarially-trained teacher, ARD~\citep{goldblum2020adversarially} proposes to modify the KL term from equation~\eqref{eq:std-distillation} to evaluate the student on an adversarially-perturbed input.
IAD~\citep{zhu2022reliable}, instead, replaces the cross-entropy loss with a KL term between the student's perturbed and unperturbed outputs, weighted according to the teacher's inability to correctly classify the adversarial input.
RSLAD~\citep{zi2021revisiting} removes the dynamic weighting and trains with two KL terms between teacher and student: one using clean inputs, the other where the student is evaluated on a perturbation seeking to maximize the distance from the clean teacher output.
In this context, less attention has been devoted to feature-space distillation methods~\citep{muhammad2021mixacm}.
CRD~\citep{vaishnavi2022feasibility} presented a logit-space distillation loss designed to transfer probabilistic certified robustness, relying on a sum of the student cross-entropy loss with an $\ell_2$ term on the models softmax outputs, all evaluated on random Gaussian perturbations.
The focus of all the above works is to transfer some of the performance of larger and more capable models onto smaller and less expensive architectures. 
However, it was shown that training a network for a given task, and then re-using it to distill knowledge onto a second model with the same architecture and training goal can improve performance through regularization.
This was observed both in the context of standard training~\citep{furlanello2018born}, and of pure adversarial training for empirical \linebreak robustness~\citep{chen2020robust}.
Finally, closely-related works at the intersection between knowledge distillation and deterministic certified robustness include: 
(i) investigating how robustness guarantees obtained on a student can be transferred from a distilled student to its teacher~\citep{indri2024distillation}, and
(ii) distilling a non-robust perceptual similarity metric onto a 1-Lipschitz architecture to obtain a certifiably-robust perceptual similarity metric~\citep{ghazanfari2024lipsim}.
In this work, targeting supervised classification tasks from the robust vision literature and using techniques based on convex relaxations, we aim to leverage an empirically-robust yet hard-to-verify teacher to improve the deterministic certified robustness of a student with~the~same~architecture.

\section{Interval Bound Propagation} \label{sec:ibp}

We here provide details concerning the computations of IBP bounds omitted from $\S$\ref{sec:nn-verification}, and define the IBP-R regularizer.
Assuming a feed-forward network structure, let $W^j$ and $b^j$ respectively be the weight and bias of the $j$-th layer, and let $\sigma$ denote a monotonic element-wise activation function, which, as stated in $\S$\ref{sec:background}, we assume to be the ReLU activation in this work.
Let us denote by ${}_jI^{n} \in \mathbb{R}^{n \times n}$ a matrix whose $j$-th column is the unit vector and filled with zeros otherwise.
Furthermore, let  $\tilde{W}^n := \left( {}_yI^{k} - I^{k} \right) W^n$ and $\tilde{\mbf{b}}^n := \left( {}_yI^{k} - I^{k} \right)\mbf{b}^n$ denote the weight and bias corresponding to the composition of the last network layer and of the difference between logits.

\paragraph{IBP pass}
IBP proceeds by first computing $\lbhat^{1} =W^{1} \xb - \epsilon \left|W^{1}\right| \mathbf{1} + \bb^{1}$ and $\ubhat^{1} = W^{1} \xb + \epsilon \left|W^{1}\right| \mathbf{1} + \bb^{1}$, and then derives $\zdifflb$ by iteratively computing lower and upper bounds ($\lbhat^{j}$, and $\ubhat^{j}$, respectively) to the network pre-activations at layer $j$:
\begin{equation}
	\label{eq:ibp}
	\begin{aligned}
		&\begin{array}{r}\hspace{-2pt}\zdifflb = [\tilde{W}^n]_+\ \sigma\left(\lbhat^{n-1}\right) + [\tilde{W}^n]_-\ \sigma\left(\ubhat^{n-1}\right) + \tilde{\bb}^n,\hspace{-2pt}\end{array} \hspace{5pt} \text{where:}\\
		&
		\hspace{-15pt} \left[\hspace{-3pt}
		\begin{array}{r}
			\lbhat^{j} = [W^{j}]_+\ \sigma\left(\lbhat^{j-1}\right) + \frac{1}{2} [W^{j}]_-\ \sigma\left(\ubhat^{j-1}\right) + \bb_{j} \\
			\ubhat^{j} = [W^{j}]_+\ \sigma\left(\ubhat^{j-1}\right) + \frac{1}{2} [W^{j}]_-\ \sigma\left(\lbhat^{j-1}\right) + \bb_{j} 
		\end{array} \hspace{-3pt}\right] \forall \ j \in \left\llbracket2,n-1\right\rrbracket.
	\end{aligned}
\end{equation}

\paragraph{IBP-R regularizer}
Given pre-activation bounds from equation~\eqref{eq:ibp}, and denoting the batch size as $B$, we compute the IBP-R regularizer $\mathcal{A}^{\mathcal{C}_{\epsilon}} (\xb, y)$~\cite{DePalma2022} as:
\[\mathcal{A}^{\mathcal{C}_{\epsilon}} (\xb, y):=B \sum_{j=1}^{n-1} \left[\ubhat^j\right]_+^T \left[\lbhat^j\right]_-.\]

\section{Proofs of technical results} \label{sec:app-proofs}

We here provide the proofs omitted from $\S$\ref{sec:method}.

\vspace{10pt}
\begin{remark}
	Let $f : \mathbb{R}^n \rightarrow \mathbb{R}$. If $\mathcal{A} \neq \emptyset$ and $\mathcal{A} \subseteq \mathcal{B}$, then $\max_{\xb \in \mathcal{A}} f(\xb) \leq \max_{\xb \in \mathcal{B}} f(\xb)$.
	\label{prop:app-remark}
\end{remark}
\vspace{10pt}

\ubloss*
\vspace{-10pt}
\begin{proof}
	Let us start by upper bounding the worst-case distillation loss $\distillationloss$:
	\begin{equation*}
		\begin{aligned}
			\distillationloss &= 
			\max_{\xb' \in \mathcal{C}_{\epsilon}(\xb)} \left\lVert   h (\xb') - \tlatent \right\rVert^2_2 
			= \max_{\xb' \in \mathcal{C}_{\epsilon}(\xb)} \left[\sum_i \left(   h (\xb')_i - \tlatent_i \right)^2 \right]\\
			& \leq \sum_i \max_{\xb' \in \mathcal{C}_{\epsilon}(\xb)} \left(   h (\xb')_i - \tlatent_i \right)^2.
		\end{aligned}
	\end{equation*}
    \looseness=-1
    Let $\mathcal{M}_i = \left\{h (\xb')_i\ \Big|\ \xb' \in \argmax_{\xb' \in \mathcal{C}_{\epsilon}(\xb)} \left(   h (\xb')_i - \tlatent_i \right)^2 \right\}$.
	By the definition of the IBP feature bounds $\latentdifflb$ and $\latentdiffub$: 
	\begin{equation*}
		 \mathcal{M}_i \subseteq \left[\latentdifflb_i, \latentdiffub_i\right].
	\end{equation*}
	Recalling remark~\ref{prop:app-remark}, we can hence further bound $\distillationloss$ as follows:
	\begin{equation*}
		\begin{aligned}
			\distillationloss  \leq \sum_i \max_{\xb' \in \mathcal{C}_{\epsilon}(\xb)} \left(   h (\xb')_i - \tlatent_i \right)^2 &= \sum_i \max_{h \in \mathcal{M}_i} \left(   h - \tlatent_i \right)^2 \\ 
			& \leq \sum_i \max_{h \in \left[\latentdifflb_i, \latentdiffub_i\right]} \left( h - \tlatent_i \right)^2.
		\end{aligned}
	\end{equation*}
	Let us note that (for any $b, l, u \in \mathbb{R}$, $l \leq u$):
	\begin{equation}
		\max_{x \in \left[l, u\right]} \left( x - b \right)^2 = \max \left\{\left( l - b \right)^2, \left( u - b \right)^2 \right\}.
		\label{eq:loss-closed-form}
	\end{equation}
	We can hence write $\distillationloss \leq \distillationlossub$.
\end{proof}
\vspace{-5pt}

\ccloss*
\vspace{-10pt}
\begin{proof}	
	Let us define $\mathcal{I}^{\alpha}_i := \left[\latentdiffcclb_i, \latentdiffccub_i\right]$.
	Owing to equation~\eqref{eq:loss-closed-form}, we can write the following alternative definition of $\distillationlosscc{\alpha}$:
	\begin{equation}
		\distillationlosscc{\alpha} = \sum_i \max_{h \in \mathcal{I}^{\alpha}_i} \left( h - \tlatent_i \right)^2.
		\label{eq:distillationlosscc-alternative}
	\end{equation}
	
	Let us recall that $\latentdiffccub := (1 - \alpha)\ h (\xb_{\text{adv}}) + \alpha\ \latentdiffub$, and \linebreak
	$\latentdiffcclb := (1 - \alpha)\ h (\xb_{\text{adv}}) + \alpha\ \latentdifflb$.
	By pointing out that $ h (\xb_{\text{adv}}) \in \left[\latentdifflb, \latentdiffub\right]$, we can see that, $\forall \alpha \in [0, 1]$:
	\begin{equation*}
			\latentdifflb \leq \latentdiffcclb \leq h (\xb_{\text{adv}}), \quad
			h (\xb_{\text{adv}}) \leq \latentdiffccub \leq \latentdiffub. 
	\end{equation*}
	Consequently, $\forall \alpha \in [0, 1]$ and $\forall i \in \left\llbracket1,m\right\rrbracket$ ($m$ being the dimensionality of the feature space):
	\begin{equation*}
		 \mathcal{I}^{0}_i  = \{h (\xb_{\text{adv}})_i\} \subseteq \mathcal{I}^{\alpha}_i \subseteq \left[\latentdifflb_i, \latentdiffub_i\right] = \mathcal{I}^{1}_i.
	\end{equation*}
	And hence, as per remark~\ref{prop:app-remark}:
	\begin{equation}
			\sum_i \max_{h \in \mathcal{I}^{0}_i} \left( h - \tlatent_i \right)^2 \leq
			\sum_i \max_{h \in \mathcal{I}^{\alpha}_i} \left( h - \tlatent_i \right)^2 \leq \sum_i \max_{h \in \mathcal{I}^{1}_i} \left( h - \tlatent_i \right)^2.
		\label{eq:proof-inequality}
	\end{equation}
	
	Using equation~\eqref{eq:distillationlosscc-alternative} with $\alpha=0$ we can see that:
	\begin{equation*}
		\begin{aligned}
			\distillationlosscc{0} &= \sum_i \max_{h \in \mathcal{I}^{0}_i} \left( h - \tlatent_i \right)^2 = \sum_i \left( h (\xb_{\text{adv}})_i - \tlatent_i \right)^2 \\
			& = \left\lVert   h (\xb_{\text{adv}}) - \tlatent \right\rVert^2_2 = \distillationlosslb.
		\end{aligned}
	\end{equation*}
	Using it again with $\alpha=1$ and recalling equation~\eqref{eq:loss-closed-form}, we obtain:
	\begin{equation*}
		\begin{aligned}
			\distillationlosscc{1} &= \sum_i \max_{h \in \mathcal{I}^{1}_i} \left( h - \tlatent_i \right)^2 = \\
			&= \sum_i \max \left\{\left(\latentdifflb_i - \tlatent_i \right)^2 , \left(\latentdiffub_i - \tlatent_i \right)^2\right\} = \distillationlossub,
		\end{aligned}
	\end{equation*}
	Hence proving the first point of the proposition.
	
	The second point can be now proved by again using equation~\eqref{eq:distillationlosscc-alternative} within equation~\eqref{eq:proof-inequality}:
	\begin{equation*}
			\distillationlosslb = \distillationlosscc{0} \leq
			\distillationlosscc{\alpha} \leq \distillationlosscc{1} = \distillationlossub.
	\end{equation*}
	
	Concerning the third point, the continuity of $\distillationlosscc{\alpha}$ for $\alpha \in [0, 1]$ can be proved by pointing out that $\distillationlosscc{\alpha}$ is composed of operators preserving continuity (pointwise $\max$), linear combinations and compositions of continuous functions of $\alpha$.
	In order to prove monotonicity, note that, for any $\alpha, \alpha' \in [0, 1]$ with $\alpha \leq \alpha'$, $\bar{h}_{\text{cc}}^{\mathcal{C}_{\epsilon}}(\alpha'; \xb) \geq \bar{h}_{\text{cc}}^{\mathcal{C}_{\epsilon}}(\alpha; \xb)$ and $\ubar{h}_{\text{cc}}^{\mathcal{C}_{\epsilon}}(\alpha'; \xb) \leq \ubar{h}_{\text{cc}}^{\mathcal{C}_{\epsilon}}(\alpha; \xb)$, implying $\mathcal{I}^{\alpha}_i \subseteq \mathcal{I}^{\alpha'}_i$.
	Hence, recalling remark~\ref{prop:app-remark}, for any $\alpha, \alpha' \in [0, 1]$ with $\alpha \leq \alpha'$, we have:
	\begin{equation*}
		\distillationlosscc{\alpha} = \sum_i \max_{h \in \mathcal{I}^{\alpha}_i} \left( h - \tlatent_i \right)^2 \leq \sum_i \max_{h \in \mathcal{I}^{\alpha'}_i} \left( h - \tlatent_i \right)^2 = \distillationlosscc{\alpha'},
	\end{equation*}
	which proves that $\distillationlosscc{\alpha}$ is monotonically increasing for $\alpha \in [0, 1]$.
	% NOTE: the algebraic proof of monotonicity is extremely tedious, as it requires enumerating 4 cases. This is more coincise.
\end{proof}

\convexcombination*
\vspace{-10pt}
\begin{proof}	
	We can write the student logit differences as: $\zb_{f} (\xb, y) =  \tilde{W}^{n} h (\xb) + \tilde{\mbf{b}}^{n}$.
	Hence, given IBP bounds to the student features $\latentdifflb$ and $\latentdiffub$ (corresponding to $\sigma\left(\lbhat^{n-1}\right)$ and $\sigma\left(\ubhat^{n-1}\right)$ in equation~\eqref{eq:ibp}, respectively), we can compute the logit differences lower bounds $\zdifflb$ as per equation~\eqref{eq:ibp}:
	\begin{equation*}
		\zdifflb = [\tilde{W}^n]_+ \latentdifflb + [\tilde{W}^n]_- \latentdiffub + \tilde{\bb}^n.
	\end{equation*}
	\looseness=-1
	Let us write the adversarial logit differences $\zb_{f}(\xb_{\text{adv}}, y)$ as a function of the adversarial student~features:
	\begin{equation*}
		\zb_{f}(\xb_{\text{adv}}, y) = \tilde{W}^n h (\xb_{\text{adv}})  + \tilde{\bb}^n = \left([\tilde{W}^n]_+ +  [\tilde{W}^n]_-\right) h (\xb_{\text{adv}})  + \tilde{\bb}^n.
	\end{equation*}
	Replacing the above two equations in the definition of the CC-IBP convex combinations $\zdiffcc$ we get:
	\begin{equation*}
		\begin{aligned}
			&\zdiffcc = ( 1 - \alpha)\ \zb_{f}(\xb_{\text{adv}}, y) + \alpha\ \zdifflb \\
			&\qquad= ( 1 - \alpha) \left[\left([\tilde{W}^n]_+ +  [\tilde{W}^n]_-\right) h (\xb_{\text{adv}})\right] + \alpha \left( [\tilde{W}^n]_+ \latentdifflb + [\tilde{W}^n]_- \latentdiffub\right) + \tilde{\bb}^n \\
			&\qquad= [\tilde{W}^n]_+ \left[ ( 1 - \alpha) h (\xb_{\text{adv}})  + \alpha \latentdifflb \right] + [\tilde{W}^n]_- \left[ ( 1 - \alpha) h (\xb_{\text{adv}})  + \alpha \latentdiffub \right] + \tilde{\bb}^n.
		\end{aligned}
	\end{equation*}
	Recalling that $\latentdiffccub := (1 - \alpha)\ h (\xb_{\text{adv}}) + \alpha\ \latentdiffub$, and \linebreak
	$\latentdiffcclb := (1 - \alpha)\ h (\xb_{\text{adv}}) + \alpha\ \latentdifflb$, we get:
	\begin{equation*}
		\zdiffcc = [\tilde{W}^n]_+ \latentdiffcclb + [\tilde{W}^n]_- \latentdiffccub + \tilde{\bb}^n,
	\end{equation*}
	which concludes the proof.
\end{proof}

\section{Pseudo-code} \label{sec:pseudo-code}

Let us denote the Euclidean projection of $\mbf{a}$ on set $\mathcal{A}$ by $\text{Proj}(\mbf{a}, \mathcal{A})$.
We here provide pseudo-code for CC-Dist and MTL-Dist (algorithm~\ref{alg:ccdist}), SABR-Dist (algorithm~\ref{alg:sabrdist}), and IBP-R-Dist (algorithm~\ref{alg:ibprdist}). 

\begin{minipage}{.98\textwidth}
	\begin{algorithm}[H]
		\caption{CC-Dist/MTL-Dist training losses $\mathcal{L}^{\text{CC-Dist}/\text{MTL-Dist}} (\alpha, \nicefrac{5}{w}; \xb, y)$}\label{alg:ccdist}
		\begin{algorithmic}[1]
			\State{\textbf{Input:} Student $f = g \circ h$, teacher feature map $h^t$, loss function $\mathcal{L}$, data point $(\xb, y)$, perturbation set $\mathcal{C}_{\epsilon}(\xb)$, hyper-parameter $\alpha$} \vspace{5pt}
			\State Compute $\xb_{\text{adv}}$ via an adversarial attack on $\mathcal{C}_{\epsilon}(\xb)$
			\State Compute $\latentdifflb$, $\latentdiffub$ and $\zdifflb$ via IBP
			\State $\latentdiffccub \leftarrow (1 - \alpha)\ h (\xb_{\text{adv}}) + \alpha\ \latentdiffub$
			\State $\latentdiffcclb \leftarrow (1 - \alpha)\ h (\xb_{\text{adv}}) + \alpha\ \latentdifflb$
			\State $\distillationlosscc{\alpha} \leftarrow  \sum_i \max \left\{
			\left(
			\latentdiffcclb_i -  \tlatent_i
			\right)^2, 
			\left(
			\latentdiffccub_i -  \tlatent_i
			\right)^2
			\right)$ \vspace{2pt}
			\If{CC-Dist}
			\State $\zb_{f}(\xb_{\text{adv}}, y) \leftarrow \left( g\left(h(\xb_{\text{adv}})\right)_y \mbf{1}^k - g\left(h(\xb_{\text{adv}})\right)\right)$
			\State $L \leftarrow \mathcal{L}\left( -\left[( 1 - \alpha)\ \zb_{f}(\xb_{\text{adv}}, y) + \alpha\ \zdifflb\right], y\right)$
			\ElsIf{MTL-Dist}
			\State $L \leftarrow (1- \alpha) \mathcal{L} (f(\xb_{\text{adv}}), y) + \alpha  \mathcal{L} (-\zdifflb, y) $
			\EndIf
			\State \textbf{Return:} $L + \nicefrac{5}{w}\ \distillationlosscc{\alpha}$

		\end{algorithmic}
	\end{algorithm}
\end{minipage}%

\begin{minipage}{.98\textwidth}
	\begin{algorithm}[H]
		\caption{SABR-Dist training loss $\mathcal{L}^{\text{SABR-Dist}} (\alpha, \nicefrac{5}{w}; \xb, y)$}\label{alg:sabrdist}
		\begin{algorithmic}[1]
			\State{\textbf{Input:} Student $f = g \circ h$, teacher feature map $h^t$, loss function $\mathcal{L}$, data point $(\xb, y)$, perturbation set $\mathcal{C}_{\epsilon}(\xb)$, hyper-parameter $\alpha$} \vspace{5pt}
			\State Compute $\xb_{\text{adv}}$ via an adversarial attack on $\mathcal{C}_{\epsilon}(\xb)$
			\State $\xb_{\alpha} \leftarrow \text{Proj}(\xb_{\text{adv}}, \mathcal{C}_{\epsilon - \lambda \epsilon}(\xb))$
			\State $\mathcal{C}_{\epsilon}^{\alpha} (\xb) \leftarrow \mathcal{C}_{\alpha \epsilon}(\xb_{\alpha})$
			\State Compute $\ubar{h}^{\mathcal{C}_{\epsilon}^{\alpha}}(\xb)$, $\bar{h}^{\mathcal{C}_{\epsilon}^{\alpha}}(\xb)$ and $\sabrzdifflb{\alpha}$ via IBP
			\State $\distillationlosssabr \leftarrow  \sum_i \max \left\{
			\left(
			\ubar{h}^{\mathcal{C}_{\epsilon}^\alpha}(\xb)_i -  \tlatent_i
			\right)^2, 
			\left(
			\bar{h}^{\mathcal{C}_{\epsilon}^\alpha}(\xb)_i -  \tlatent_i
			\right)^2
			\right)$ \vspace{2pt}
			\State \textbf{Return:} $\mathcal{L}\left( -\sabrzdifflb{\alpha}, y\right) + \nicefrac{5}{w}\ \distillationlosssabr$
		\end{algorithmic}
	\end{algorithm}
\end{minipage}%

\begin{minipage}{.98\textwidth}
	\begin{algorithm}[H]
		\caption{IBP-R-Dist training loss $\mathcal{L}^{\text{IBP-R-Dist}} (\lambda, \nicefrac{5}{w}; \xb, y)$}\label{alg:ibprdist}
		\begin{algorithmic}[1]
			\State{\textbf{Input:} Student $f = g \circ h$, teacher feature map $h^t$, loss function $\mathcal{L}$, data point $(\xb, y)$, perturbation set $\mathcal{C}_{\epsilon}(\xb)$, hyper-parameter $\alpha$, batch size $B$} \vspace{5pt}
			\State Compute $\xb_{\text{adv}}$ via an adversarial attack on $\mathcal{C}_{\epsilon}(\xb)$
			\State Compute pre-activation bounds $\lbhat^j$, $\ubhat^j$ for $j \in \left\llbracket1,n-1\right\rrbracket$ via IBP
			\State $\mathcal{A}^{\mathcal{C}_{\epsilon}} (\xb, y) \leftarrow B \sum_{j=1}^{n-1} \left[\ubhat^j\right]_+^T \left[\lbhat^j\right]_- $
			\State $\distillationlosslb \leftarrow \left\lVert   h (\xb_{\text{adv}}) - \tlatent \right\rVert^2_2$
			\State \textbf{Return:} $\mathcal{L} (f(\xb_{\text{adv}}), y) + \lambda \mathcal{A}^{\mathcal{C}_{\epsilon}} (\xb, y) + \nicefrac{5}{w}\ \distillationlosslb$
		\end{algorithmic}
	\end{algorithm}
\end{minipage}%
\section{Experimental details}  \label{sec:app-experiments-details}

This appendix provides experimental details omitted from $\S$\ref{sec:experiments}.

\subsection{Experimental setup} \label{sec:exp-setup}

Experiments were each carried out on an internal cluster using a single GPU and employing from $6$ to $12$ CPU cores.
GPUs belonging to the following models were employed: Nvidia GTX 1080 Ti, Nvidia RTX 2080 Ti, Nvidia RTX 6000, Nvidia RTX 8000, Nvidia V100, Nvidia A100.
CPU memory was capped at $10$ GB for the training experiments, and at $100$ GB for the verification experiments.
The runtime of training runs ranges from roughly 5 hours for CIFAR-10, to roughly 4 days for downscaled ImageNet.
For verification experiments, runtime ranges from below 8 hours for CIFAR-10 with $\epsilon=\nicefrac{8}{255}$ to roughly 12 days for downscaled ImageNet.

\subsection{Training setup and hyper-parameters}

We now report the details of the employed training setup and hyper-parameters.

\subsubsection{Datasets}
As stated in $\S$\ref{sec:main-exp}, the experiments were carried out on the following datasets: CIFAR-10~\citep{CIFAR10}, TinyImageNet~\citep{Le2015}, and downscaled ImageNet ($64\times64$)~\citep{Chrabaszcz2017}.
For all three, we train using random horizontal flips, random crops, and normalization as done in~\citep{DePalma2024}.
The employed dataset splits comply with previous work in the area~\citep{DePalma2024,Mao2023,Mueller2023}.
Specifically, for CIFAR-10, which consists of $32\times32$ RGB images divided into $10$ classes, we trained on the original $50,000$ training points, and evaluated on the original $10,000$ test points.
For TinyImageNet, consisting of $64\times64$ RGB images and $200$ classes, we trained on the original $100,000$ training points, and evaluated on the original $10,000$ validation images.
Finally, for downscaled ImageNet, which consists of $64\times64$ RGB images and $1000$ classes, we employed the original training set of $1,281,167$ images, and evaluated on the original $50,000$ validation images.

\subsubsection{Training schedules and implementation details} \label{sec:training-schedules}

For distilled models, their teachers and their baselines, trained networks are initialized using the specialized scheme by~\citet{Shi2021}, which reduces the magnitude of the IBP bounds. 
Teachers are trained using pure adversarial training, relying on a 10-step PGD attack~\citep{Madry2018} with step size $\eta=0.25 \epsilon$.
Complying with~\citet{DePalma2024}, $\xb_{\text{adv}}$ is computed using a single-step PGD attack with uniform random initialization and step size $\eta=10.0 \epsilon$ on all setups except for CIFAR-10 with $\epsilon=\nicefrac{2}{255}$, which instead relies on a $8$-step PGD attack with step size $\eta=0.25 \epsilon'$ run on a larger effective input perturbation ($\epsilon' = 2.1 \epsilon$).
All methods use a batch size of $128$ throughout the experiments, and gradient clipping with norm equal to $10$.
We do not perform any early stopping.

\looseness=-1
As relatively common in the adversarial training literature~\citep{Andriushchenko2020,Jorge2022}, teachers are trained using the SGD optimizer (with momentum set to $0.9$), shorter schedules and a cyclic learning rate, which linearly increases the learning rate from $0$ to $0.2$ during the first half of the training, and then decreases it back to $0$. As a longer number of teacher training epochs was found to be beneficial on TinyImageNet, we report it among the hyper-parameters in table~\ref{fig:hyperparameters}.
As instead common in the certified training literature, student models and their baselines are trained using the Adam optimizer~\citep{Kingma2015} with longer training schedules and a learning rate of $5 \times 10^{-4}$, which is decayed twice by a factor $0.2$.
Their training starts with one epoch of ``warm-up'', where the certified training loss is replaced by the standard cross entropy on the clean inputs, and proceeds with a ``ramp-up'' phase, during which the training perturbation radius $\epsilon$ is gradually increased from $0$ to its target value, and the IBP regularization from~\citet{Shi2021} is added to the training loss (with coefficient $\lambda=0.5$ on CIFAR-10, and $\lambda=0.2$ for TinyImageNet and downscaled ImageNet).
Benchmark-specific details, which are in accordance with~\citet{DePalma2024}, follow:
\begin{itemize}[leftmargin=1.3em]
	\item CIFAR-10 with $\epsilon=\nicefrac{2}{255}$: the training schedule for students and baselines is $160$-epochs long, with $80$ epochs of ramp-up, and with the learning rate decayed at epochs $120$ and $140$.
	\item CIFAR-10 with $\epsilon=\nicefrac{8}{255}$: $260$ epochs, with $80$ epochs of ramp-up, and with the learning rate decayed at epochs $180$ and $220$.
	\item TinyImageNet: $160$ epochs, with $80$ ramp-up epochs, and decaying the learning rate at epochs $120$~and~$140$.
	\item Downscaled ImageNet:  $80$ epochs, with $20$ epochs of ramp-up, and decaying the learning rate at epochs $60$ and $70$.
\end{itemize}

\looseness=-1
In order to separate the effect of distillation from any other potential implementation detail, the baseline evaluations (CC-IBP, MTL-IBP, SABR, and IBP-R) from tables~\ref{fig:ccdist-results}~and~\ref{fig:sabr-distillation} were obtained by employing the student implementation with $\beta=0$.

\subsubsection{Network architecture}

The employed \texttt{CNN-7} architecture is left unvaried with respect to previous work~\citep{DePalma2024,Mueller2023,Mao2023,mao2024ctbench} except on downscaled ImageNet, where we applied a small modification to the linear layers.
Specifically, in order to make sure that the feature space employed for distillation is larger than the network output space, we set the output size of the penultimate layer (and the input size of the last) equal to $1024$, instead of $512$ as for the other datasets and in~\citet{DePalma2024}.
As originally suggested by~\citet{Shi2021} to improve performance, and complying with previous work~\citep{DePalma2024,Mueller2023,Mao2023,mao2024ctbench}, we employ batch normalization (BatchNorm) after every network layer except the last.
In our implementation, adversarial attacks are carried out with the network in evaluation mode, hence using the current BatchNorm running statistics.
Except during the warm-up phase, where we also perform an evaluation on the unperturbed data points, the network is exclusively evaluated on the computed adversarial inputs $\xb_{\text{adv}}$, which hence dominate the computed running statistics for most of the training and for the final network.
Finally, except for the clean loss during warm-up, which is computed using the unperturbed current batch statistics,
all the training loss computations (including those requiring IBP bounds) are carried out using the batch statistics from the computed adversarial inputs (hence in training mode).

\subsubsection{Hyper-parameters} \label{sec:hyperparameters}

\begin{table}[t!]
%	\vspace{-7pt}
	\sisetup{detect-all = true}
	
	\centering
	
	\scriptsize
	\setlength{\tabcolsep}{4pt}
	\aboverulesep = 0.3mm  % 0.605mm
	\belowrulesep = 0.1mm  % 0.984mm
	\caption{\small 
		\looseness=-1
		Hyper-parameter settings for the models from tables~\ref{fig:ccdist-results}~and~\ref{fig:sabr-distillation}, and their respective teachers.
		\label{fig:hyperparameters}}
	\vspace{5pt}
	\begin{adjustbox}{width=.81\textwidth, center}
		
		\begin{tabular}{
				c
				c
				l
				c
				c
				c
				c
			} 
			
			\toprule \\[-5pt]
			
			\multicolumn{1}{c}{\multirow{1}{*}{Dataset}} &
			\multicolumn{1}{c}{\multirow{1}{*}{$\epsilon$}} &
			\multicolumn{1}{c}{\multirow{1}{*}{Method}} &
			\multicolumn{1}{c}{$\alpha$ or $\lambda$}  &
			\multicolumn{1}{c}{$\ell_1$} & 
			\multicolumn{1}{c}{Teacher n. epochs} & 
			\multicolumn{1}{c}{Teacher $\ell_1$} 
			\\
			
			\cmidrule(lr){1-2} \cmidrule(lr){3-3} \cmidrule(lr){4-5} \cmidrule(lr){6-7} \\[-5pt]
			
			\multirow{14.5}{*}{CIFAR-10} & \multirow{8}{*}{$\frac{2}{255}$} &
			
			\multicolumn{1}{ c }{\textsc{CC-Dist}}
			& $10^{-2}$ &  $3 \times 10^{-6}$ & 30 & $10^{-5}$   \\
			
			& & \multicolumn{1}{ c }{\textsc{CC-IBP}} 
			& $10^{-2}$  &  $3 \times 10^{-6}$ & / & /\\ 
			
			\cmidrule(lr){3-7} \\[-5pt]
			
			& & \multicolumn{1}{ c }{\textsc{SABR-Dist}}
			& $10^{-2}$ &  $3 \times 10^{-6}$ & 30 & $10^{-5}$   \\
			
			& & \multicolumn{1}{ c }{\textsc{SABR}} 
			& $10^{-2}$  &  $3 \times 10^{-6}$ & / & /\\ 
			
			\cmidrule(lr){3-7} \\[-5pt]
			
			& & \multicolumn{1}{ c }{\textsc{MTL-Dist}}
			& $4 \times 10^{-3}$ &  $3 \times 10^{-6}$ & 30 & $10^{-5}$   \\
			
			& & \multicolumn{1}{ c }{\textsc{MTL-IBP}} 
			& $4 \times 10^{-3}$  &  $3 \times 10^{-6}$ & / & /\\ 
			
			\cmidrule(lr){3-7} \\[-5pt]
			
			& & \multicolumn{1}{ c }{\textsc{IBP-R-Dist}}
			& $3 \times 10^{-4}$ &  $3 \times 10^{-6}$ & 30 & $2\times10^{-5}$   \\
			
			& & \multicolumn{1}{ c }{\textsc{IBP-R}} 
			& $2 \times 10^{-4}$  &  $3 \times 10^{-6}$ & / & /\\ 
			
			\cmidrule(lr){2-7} \\[-5pt]
			
			& \multirow{6}{*}{$\frac{8}{255}$} &
			
			\multicolumn{1}{ c }{\textsc{CC-Dist}} 
			& $0.5$ &  $0$ & 30 & $5 \times 10^{-6}$  \\
			
			& & \multicolumn{1}{ c }{\textsc{CC-IBP}} 
			& $0.5$ &  $0$ & / & /  \\
			
			\cmidrule(lr){3-7} \\[-5pt]
			
			& & \multicolumn{1}{ c }{\textsc{SABR-Dist}} 
			& $0.7$ &  $0$ & 30 & $1 \times 10^{-6}$  \\
			
			& & \multicolumn{1}{ c }{\textsc{SABR}} 
			& $0.7$ &  $0$ & / & /  \\
			
			\cmidrule(lr){3-7} \\[-5pt]
			
			& & \multicolumn{1}{ c }{\textsc{MTL-Dist}} 
			& $0.5$ &  $1 \times 10^{-7}$ & 30 & $5 \times 10^{-6}$  \\
			
			& & \multicolumn{1}{ c }{\textsc{MTL-IBP}} 
			& $0.5$ &  $1 \times 10^{-7}$ & / & /  \\
			
			\cmidrule{1-7} \\[-5pt]
			
			\multirow{8}{*}{TinyImageNet} & \multirow{8}{*}{$\frac{1}{255}$} &
			
			\multicolumn{1}{ c }{\textsc{CC-Dist}} 
			& $5 \times 10^{-3}$ &   $5 \times 10^{-5}$ & 100 & $5 \times 10^{-5}$   \\
			
			& & \multicolumn{1}{ c }{\textsc{CC-IBP}}
			& $3 \times 10^{-3}$ &   $5 \times 10^{-5}$ & / & / \\
			
			\cmidrule(lr){3-7} \\[-5pt]
			
			& & \multicolumn{1}{ c }{\textsc{SABR-Dist}} 
			& $5 \times 10^{-3}$ &   $5 \times 10^{-5}$ & 100 & $5 \times 10^{-5}$   \\
			
			& & \multicolumn{1}{ c }{\textsc{SABR}}
			& $4 \times 10^{-3}$ &   $5 \times 10^{-5}$ & / & / \\
			
			\cmidrule(lr){3-7} \\[-5pt]
			
			& & \multicolumn{1}{ c }{\textsc{MTL-Dist}} 
			& $3 \times 10^{-3}$ &   $5 \times 10^{-5}$ & 100 & $5 \times 10^{-5}$   \\
			
			& & \multicolumn{1}{ c }{\textsc{MTL-IBP}}
			& $2 \times 10^{-3}$ &   $5 \times 10^{-5}$ & / & / \\
			
			\cmidrule(lr){3-7} \\[-5pt]
			
			& & \multicolumn{1}{ c }{\textsc{IBP-R-Dist}} 
			& $2 \times 10^{-5}$ &   $5 \times 10^{-5}$ & 100 & $5 \times 10^{-5}$   \\
			
			& & \multicolumn{1}{ c }{\textsc{IBP-R}}
			& $1 \times 10^{-5}$ &   $5 \times 10^{-5}$ & / & / \\
			
			\cmidrule{1-7} \\[-5pt]
			
			\multirow{2}{*}{ImageNet64} & \multirow{2}{*}{$\frac{1}{255}$} &
			
			\multicolumn{1}{ c }{\textsc{CC-Dist}} 
			&  $5 \times 10^{-3}$ &   $10^{-5}$ & 30 & $5 \times 10^{-6}$  \\ 
			
			& & \multicolumn{1}{ c }{\textsc{CC-IBP}}
			&  $5 \times 10^{-3}$ &   $10^{-5}$ & / & /  \\  
			
			\bottomrule 
			
		\end{tabular}
	\end{adjustbox}
\end{table}

As done in previous work~\citep{Gowal2019,Zhang2020,Shi2021,Mueller2023,DePalma2024,mao2024ctbench}, and hence ensuring a fair comparison in table~\ref{fig:ccdist-literature} within appendix~\ref{sec:exp-sota}, tuning was carried out directly on the evaluation sets.
Table~\ref{fig:hyperparameters} reports a list of the main method and regularization hyper-parameter values for our evaluations from tables~\ref{fig:ccdist-results}~and~\ref{fig:sabr-distillation}. 
The details of the employed training schedules, taken from~\citet{DePalma2024}, are instead reported in appendix~\ref{sec:training-schedules}.

\paragraph{CC-Dist}
As explained in $\S$\ref{sec:ccdist}, we advocate for a constant distillation coefficient for CC-Dist: we noticed that $\beta=\nicefrac{5}{w}$ yielded good performance across our earlier CIFAR-10 experiments, and then decided to keep it to the same value across all evaluations.
On CIFAR-10, we did not tune the $\alpha$ coefficient for either method: it was set to be the CC-IBP $\alpha$ coefficient employed in~\citet{DePalma2024}.
Complying with common practice in the certified training literature~\citep{DePalma2024,Mueller2023,mao2024ctbench}, $\ell_1$ regularization is applied on top of the employed training losses: the corresponding values for CC-Dist and CC-IBP models were not tuned but taken from the CC-IBP values reported in the expressive losses work~\citep{DePalma2024}.
On TinyImageNet and ImageNet64, we found that both the standard performance and the certified accuracy of CC-IBP could be significantly improved compared to the original results from~\citet{DePalma2024}, which respectively employ $\alpha=10^{-2}$ and $\alpha=5 \times 10^{-2}$ for CC-IBP on TinyImageNet and ImageNet64, by decreasing $\alpha$.
In particular, we selected the smallest $\alpha$ value resulting in strictly better robustness-accuracy trade-offs (i.e., before certified accuracy started decreasing with lower $\alpha$ values).
We trained a series of potential teachers for CC-Dist with varying $\ell_1$ coefficient and number of training epochs, and selected the teacher depending on the performance of the resulting CC-Dist model. 
On all datasets except TinyImageNet, where longer schedules were beneficial to CC-Dist performance, we found a teacher trained with the relatively short $30$-epoch schedule to result in strong CC-Dist performance.
Finally, generally speaking, we advocate for the re-use of the selected CC-IBP $\alpha$ coefficient for CC-Dist.
On TinyImageNet, the only setting where the employed $\alpha$ values for the reported models differ across CC-Dist and CC-IBP, we found a larger CC-Dist $\alpha$ to result in simultaneously better standard performance and certified robustness compared to all results from the literature.
Nevertheless, re-using $\alpha = 3 \times 10^{-3}$ for CC-Dist produced a model with $44.08\%$ standard accuracy and $27.40\%$ certified accuracy (compare with table~\ref{fig:ccdist-literature}), which, albeit not improving on both metrics at once compared to the MTL-IBP results from~\citep{mao2024ctbench}, significantly improves upon the overall performance trade-offs seen in the literature. 
Unless otherwise stated, all hyper-parameters throughout the experimental evaluations comply with those reported in table~\ref{fig:hyperparameters}.

\paragraph{SABR-Dist, MTL-Dist and IBP-R-Dist} 
\looseness=-1
On CIFAR-10 with $\epsilon=\nicefrac{2}{255}$ various teachers, method-specific regularization ($\lambda$) and distillation coefficients were tested for IBP-R-Dist and IBP-R. 
On the same benchmark, we only varied $\beta$ and kept the same teacher as table~\ref{fig:teacher-student} for SABR-Dist and MTL-Dist, keeping the $\alpha$ coefficients to the values tuned in~\citet{DePalma2024} for both and for their respective baselines.
$\beta=\nicefrac{5}{w}$ was chosen for all methods, suggesting that, as for CC-Dist, it should not be treated as a hyper-parameter: we hence kept it fixed for both $\epsilon=\nicefrac{8}{255}$ experiments and TinyImageNet.
On CIFAR-10 with $\epsilon=\nicefrac{8}{255}$, we kept $\alpha$ values as in~\citet{DePalma2024}, and tested various teachers.
As this was beneficial for CC-Dist, on TinyImageNet we tuned $\alpha$ and $\lambda$ values for both distillation-based methods and their baselines, keeping the teachers fixed to the one employed for CC-Dist to minimize computational overhead.

\subsection{Verification setup} \label{sec:verif-setup}

\looseness=-1
The employed verification setup is analogous to the one from~\citet{DePalma2024}.
As stated in $\S$\ref{sec:experiments}, we use the open source OVAL verification framework~\citep{Bunel2018,BunelDP20,improvedbab}, which performs branch-and-bound, and a configuration based on alpha-beta-CROWN bounds~\citep{betacrown}.
Before running branch-and-bound, we try verifying the property via IBP and CROWN bounds from~\texttt{auto\_LiRPA}~\citep{Xu2020}, or to falsify it through a PGD attack~\citep{Madry2018}.
Differently from~\citet{DePalma2024}, we use a timeout of 600 seconds (as opposed to 1800 seconds). 
We perform verification by running the framework to compute $\min_i \zdiffopt_i$, where the $\min$ operator is converted into an equivalent auxiliary ReLU network to append to $\zb_{f}(\xb, y)$~\citep{Bunel2020,DePalma2024}.
For TinyImageNet and downscaled ImageNet, differently from~\citet{DePalma2024}, in order to reduce the size of the auxiliary network, we exclude from the $\min$ operator all the logit differences that were already proved to be positive by the IBP or CROWN bounds computed before running branch-and-bound.
Finally, the configuration employed to verify the TinyImageNet and downscaled ImageNet teacher models (results reported in table~\ref{fig:teacher-student}) computes looser pre-activation bounds (using CROWN~\citep{Zhang2018} as opposed to 5 iterations of alpha-CROWN~\citep{xu2021fast}) at the root of branch-and-bound: this was found to be a more effective verifier for networks trained via pure adversarial training.

\subsection{Software acknowledgments and licenses}

As described above, our code relies on the OVAL verification framework to verify the models, which was released under an MIT license.
The training codebase is analogous to the one from the expressive losses work~\citep{DePalma2024}, also released under an MIT license: this was in turn based on the codebase from~\citet{Shi2021}, released under a 3-Clause BSD license.
Both above repositories, and hence our codebase, rely on the~\texttt{auto\_LiRPA}~\citep{Xu2020} framework for incomplete verification, which has a 3-Clause BSD license.
Concerning datasets: downscaled ImageNet was obtained from the ImageNet website (\url{https://www.image-net.org/download.php}), and TinyImageNet from the website of the CS231n Stanford class  (\url{http://cs231n.stanford.edu/TinyImageNet-200.zip}).
MNIST and CIFAR-10 were instead downloaded using \texttt{torchvision.datasets}~\citep{Paszke2019}.

%\vspace{50pt}
\section{Supplementary experiments} \label{sec:app-supp-experiments}

This appendix reports experimental results omitted from the main paper.

\subsection{Comparison of CC-Dist with literature results} \label{sec:exp-sota}

\begin{table*}[t!]
	%	\vspace{-7pt}		
	\sisetup{detect-all = true}
	
	\setlength{\heavyrulewidth}{0.09em}
	\setlength{\lightrulewidth}{0.5em}
	\setlength{\cmidrulewidth}{0.03em}
	
	\centering
	
	\scriptsize
	\setlength{\tabcolsep}{4pt}
	\aboverulesep = 0.3mm  % 0.605mm
	\belowrulesep = 0.1mm  % 0.984mm
	\caption{%\small 
		\looseness=-1
		Comparison of CC-Dist with literature results when training for certified robustness against $\ell_\infty$ norm perturbations. % on CIFAR-10, TinyImageNet and downscaled ImageNet ($64\times64$).
		We highlight in bold the entries corresponding to the largest standard or certified accuracy for each benchmark, and, when they do not coincide, underline the best accuracies for ReLU-based architectures.
		\label{fig:ccdist-literature}}
	\begin{adjustbox}{width=.93\textwidth, center}
		\begin{tabular}{
				c
				c
				l
				c
				S[table-format=2.2]
				S[table-format=2.2]
			} 
			
			\toprule \\[-5pt]
			
			\multicolumn{1}{ c }{Dataset} &
			\multicolumn{1}{ c }{$\epsilon$} &
			\multicolumn{1}{ c }{Method} &
			\multicolumn{1}{ c }{Source} &
			\multicolumn{1}{ c }{Standard acc. [\%]} &
			\multicolumn{1}{ c }{Certified acc. [\%]}  \\
			
			\cmidrule(lr){1-2} \cmidrule(lr){2-2} \cmidrule(lr){3-4} \cmidrule(lr){5-6} \\[-5pt]
			
			\multirow{22.5}{*}{CIFAR-10} & \multirow{11}{*}{$\frac{2}{255}$} &
			
			\multicolumn{1}{ c }{\textsc{CC-Dist}} & this work
			& \B 81.55 & \B 64.60 \\
			
			& & \multicolumn{1}{ c }{\textsc{CC-IBP}} & this work
			& 79.51 & 63.50  \\
			
			\cmidrule(lr){3-6} \\[-6pt]
			
			& & \multicolumn{1}{ c }{\textsc{CC-IBP}} & \citet{DePalma2024}
			& 80.09 & 63.78 \\
			
			& & \multicolumn{1}{ c }{\textsc{MTL-IBP}$^\dagger$} & \citet{mao2024ctbench}
			& 78.82 & 64.41 \\
			
			& & \multicolumn{1}{ c }{\textsc{STAPS}} & \citet{Mao2023}
			& 79.76 & 62.98  \\
			
			& & \multicolumn{1}{ c }{\textsc{SABR}$^\dagger$} & \citet{mao2024understanding}
			& 79.89 & 63.28  \\
			
			& & \multicolumn{1}{ c }{\textsc{SortNet}} & \citet{Zhang2022rethinking}
			& 67.72 & 56.94\\
			
			& & \multicolumn{1}{ c }{\textsc{IBP-R}$^\dagger$} & \citet{mao2024understanding}
			& 80.46 &  62.03 \\
			
			& & \multicolumn{1}{ c }{\textsc{IBP}$^\dagger$} & \citet{mao2024understanding} 
			& 68.06 &   56.18 \\
			
			& & \multicolumn{1}{ c }{\textsc{CROWN-IBP}$^\dagger$} & \citet{mao2024ctbench} 
			&  67.60 &  53.97  \\ 
			
			\cmidrule(lr){2-6} \\[-6pt]
			
			& \multirow{11}{*}{$\frac{8}{255}$} &
			
			\multicolumn{1}{ c }{\textsc{CC-Dist}} & this work
			& \B 55.13 & \U{35.52} \\
			
			& & \multicolumn{1}{ c }{\textsc{CC-IBP}} & this work
			& 54.46 & 35.42  \\
			
			\cmidrule(lr){3-6} \\[-6pt]
			
			& & \multicolumn{1}{ c }{\textsc{CC-IBP}} & \citet{DePalma2024}
			& 53.71 & 35.27 \\
			
			& & \multicolumn{1}{ c }{\textsc{MTL-IBP}$^\dagger$} & \citet{mao2024ctbench}
			& 54.28 & 35.41 \\
			
			& & \multicolumn{1}{ c }{\textsc{STAPS}} & \citet{Mao2023}
			& 52.82 &  34.65  \\
			
			& & \multicolumn{1}{ c }{\textsc{SABR}$^\dagger$} & \citet{mao2024ctbench}
			&  52.71 & 35.34  \\
			
			& & \multicolumn{1}{ c }{\textsc{SortNet}} & \citet{Zhang2022rethinking}
			& 54.84 & \B 40.39 \\
			
			& & \multicolumn{1}{ c }{\textsc{IBP-R}} & \citet{DePalma2022}
			& 52.74 &   27.55  \\ 
			
			& & \multicolumn{1}{ c }{\textsc{IBP}} & \citet{Shi2021} 
			&  48.94  &  34.97 \\ 
			
			& & \multicolumn{1}{ c }{\textsc{CROWN-IBP}$^\dagger$} & \citet{mao2024ctbench} 
			&  48.25 & 32.59 \\ 
			
			\cmidrule{1-6} \\[-6pt]
			
			\multirow{10}{*}{TinyImageNet} & \multirow{10}{*}{$\frac{1}{255}$} &
			
			\multicolumn{1}{ c }{\textsc{CC-Dist}} & this work
			& \B 43.78  & \B 27.88 \\
			
			& & \multicolumn{1}{ c }{\textsc{CC-IBP}} & this work
			& 41.28 & 26.53  \\
			
			\cmidrule(lr){3-6} \\[-6pt]
			
			& & \multicolumn{1}{ c }{\textsc{Exp-IBP}} & \citet{DePalma2024}
			& 38.71 &  26.18 \\
			
			& & \multicolumn{1}{ c }{\textsc{MTL-IBP}$^\dagger$} & \citet{mao2024ctbench}
			& 35.97 & 27.73 \\
			
			& & \multicolumn{1}{ c }{\textsc{STAPS}$^\dagger$} & \citet{mao2024ctbench}
			& 30.63 &  22.31  \\
			
			& & \multicolumn{1}{ c }{\textsc{SABR}$^\dagger$} & \citet{DePalma2024}
			& 38.68 & 25.85   \\
			
			& & \multicolumn{1}{ c }{\textsc{SortNet}} & \citet{Zhang2022rethinking}
			& 25.69 & 18.18  \\
			
			& & \multicolumn{1}{ c }{\textsc{IBP}$^\dagger$} & \citet{mao2024ctbench} 
			& 26.77 &   19.82 \\
			
			& & \multicolumn{1}{ c }{\textsc{CROWN-IBP}$^\dagger$} & \citet{mao2024ctbench} 
			& 28.44 &  22.14  \\ 
			
			\cmidrule{1-6} \\[-6pt]
			
			\multirow{7.5}{*}{ImageNet64} & \multirow{7.5}{*}{$\frac{1}{255}$} &
			
			\multicolumn{1}{ c }{\textsc{CC-Dist}} & this work
			& \B 28.17 &  \B 13.96  \\ 
			
			& & \multicolumn{1}{ c }{\textsc{CC-IBP}} & this work
			& 25.94  & 13.69  \\
			
			\cmidrule(lr){3-6} \\[-6pt]
			
			& & \multicolumn{1}{ c }{\textsc{Exp-IBP}} & \citet{DePalma2024}
			& 22.73 & 13.30  \\
			
			& & \multicolumn{1}{ c }{\textsc{SABR}$^\dagger$} & \citet{DePalma2024}
			& 20.33 & 12.39  \\
			
			& & \multicolumn{1}{ c }{\textsc{SortNet}} & \citet{Zhang2022rethinking}
			& 14.79 & 9.54  \\
			
			& & \multicolumn{1}{ c }{\textsc{CROWN-IBP}} & \citet{Xu2020} 
			& 16.23 &  8.73  \\ 
			
			& & \multicolumn{1}{ c }{\textsc{IBP}} & \citet{Gowal2019} 
			& 15.96 &   6.13 \\
			
			\bottomrule 
			
			\\[-5pt]
			\multicolumn{6}{ l }{
				\shortstack[l]{$^\dagger$Evaluation from later work attaining a larger standard or certified accuracy than reported in the original work.} 
			}
		\end{tabular}
	\end{adjustbox}
	%	\vspace{-7pt}
\end{table*}

In line with common practice in the area~\citep{DePalma2024,Mao2023,Mueller2023}, we here compare the first repetition of the CC-Dist experiments from $\S$\ref{sec:main-exp} with results taken from the certified training literature.
For each method from the literature, we report the best attained performance across previous works and network architectures. 
In order to isolate the effect of any modification within our codebase and experimental setup, we report both (first-repetition) CC-IBP results from our evaluations, and the literature results corresponding to the best-performing expressive loss (CC-IBP, MTL-IBP, or Exp-IBP) from the original work~\citep{DePalma2024}.

\looseness=-1
Table~\ref{fig:ccdist-literature} shows that CC-Dist establishes a new state-of-the-art, improving on both standard and certified accuracy compared to previusly-reported results, on all benchmarks except CIFAR-10 with $\epsilon=\nicefrac{8}{255}$.
Differently from the other settings, the best certified accuracy on this benchmark is attained by SortNet~\citep{Zhang2022rethinking}, which uses a specialized 1-Lipschitz network architecture (see $\S$\ref{sec:related}). Its performance on this setup can be further improved through the use of additional generated data~\citep{Altstidl2024}, which is outside the scope of our work.
While the sub-par performance of ReLU networks on this benchmark is well known in the literature~\citep{Mueller2023,DePalma2024}, we are hopeful that distillation from empirically-robust teachers may be beneficial to 1-Lipschitz networks too, and defer the investigation to future work.
On TinyImageNet and downscaled ImageNet we found that both the standard and certified accuracies of CC-IBP benefit from a smaller $\alpha$ coefficient than those employed in~\citet{DePalma2024}, explaining the difference between our CC-IBP results and those reported in the original work.
Owing to the combined effect of this and of the benefits of our novel distillation loss, the difference between CC-Dist and the results from the literature is particularly remarkable on these two larger-scale benchmarks.

\subsection{Other special cases of the CC-Dist distillation loss} \label{sec:distillation-special-cases}

\begin{table}[b!]
%	\vspace{-7pt}
	\sisetup{detect-all = true}
	
	\centering
	
	\scriptsize
	\setlength{\tabcolsep}{4pt}
	\aboverulesep = 0.3mm  % 0.605mm
	\belowrulesep = 0.1mm  % 0.984mm
	\caption{\small 
		\looseness=-1
		Comparison between the CC-Dist models from table~\ref{fig:ccdist-literature} and other special cases of the CC-Dist parametrized distillation loss.%
		\label{fig:special-cases}}
	\vspace{5pt}
	\begin{adjustbox}{width=.85\textwidth, center}
		
		\begin{tabular}{
				c
				c
				l
				c
				c
				S[table-format=2.2]
				S[table-format=2.2]
			} 
			
			\toprule \\[-5pt]
			
			\multicolumn{1}{c}{Dataset} &
			\multicolumn{1}{c}{$\epsilon$} &
			\multicolumn{1}{c}{Method} &
			\multicolumn{1}{c}{$\beta$} & 
			\multicolumn{1}{c}{Teacher $\ell_1$} &
			\multicolumn{1}{c}{Standard acc. [\%]}  &
			\multicolumn{1}{c}{Certified acc. [\%]} \\
			
			\cmidrule(lr){1-2} \cmidrule(lr){3-3} \cmidrule(lr){4-5} \cmidrule(lr){6-7} \\[-5pt]
			
			\multirow{11.3}{*}{CIFAR-10} & \multirow{5}{*}{$\frac{2}{255}$} &
			
			\multicolumn{1}{ c }{\textsc{CC-Dist}} & $\nicefrac{5}{w}$ & $10^{-5}$
			& 81.55 &  64.60 \\
			
			& & \multicolumn{1}{ c }{\textsc{CC-Dist}$_0$}   & $\nicefrac{5}{w}$ & $10^{-5}$
			& 81.74  &  64.22 \\ 
			
			& & \multicolumn{1}{ c }{\textsc{CC-Dist}-Clean}  & $\nicefrac{5}{w}$  & $5 \times 10^{-5}$
			& 81.52  &  64.50 \\ 
			
			& & \multicolumn{1}{ c }{\textsc{CC-Dist}$_1$}  & $\nicefrac{\left(5 \times 10^{-4}\right)}{w}$ & $5 \times 10^{-5}$
			& 79.84  &  63.93 \\ 
			
			\cmidrule(lr){3-7} \\[-6pt]
			
			& & \multicolumn{1}{ c }{\textsc{CC-IBP}}  & 0 & / 
			& 79.51 & 63.50  \\ 
			
			\cmidrule(lr){2-7} \\[-5pt]
			
			& \multirow{5}{*}{$\frac{8}{255}$} &
			
			\multicolumn{1}{ c }{\textsc{CC-Dist}}  & $\nicefrac{5}{w}$ & $5 \times 10^{-6}$
			& 55.13 &  35.52 \\
			
			& & \multicolumn{1}{ c }{\textsc{CC-Dist}$_0$}  & $\nicefrac{5}{w}$ & $2 \times 10^{-5}$
			& 55.63  &  35.08 \\ 
			
			& & \multicolumn{1}{ c }{\textsc{CC-Dist}-Clean}  & $\nicefrac{2}{w}$ & $5 \times 10^{-6}$
			& 55.47  &  35.06 \\ 
			
			& & \multicolumn{1}{ c }{\textsc{CC-Dist}$_1$}  & $\nicefrac{5}{w}$ & $2 \times 10^{-6}$
			& 55.91  &  35.35 \\ 
			
			& & \multicolumn{1}{ c }{\textsc{CC-Dist}$_0$}  & $\nicefrac{2}{w}$ & $5 \times 10^{-6}$
			& 54.65  &  35.46 \\ 
			
			& & \multicolumn{1}{ c }{\textsc{CC-Dist}-Clean}  & $\nicefrac{2}{w}$ & $5 \times 10^{-6}$
			& 54.71  &  35.13 \\ 
			
			& & \multicolumn{1}{ c }{\textsc{CC-Dist}$_1$}  & $\nicefrac{2}{w}$ & $5 \times 10^{-6}$
			& 54.87  &  35.55 \\ 
			
			\cmidrule(lr){3-7} \\[-6pt]
			
			& & \multicolumn{1}{ c }{\textsc{CC-IBP}}  & 0 & / 
			& 54.46 & 35.42 \\
			
			\bottomrule 
			
		\end{tabular}
	\end{adjustbox}
\end{table}

As seen in proposition~\ref{prop:ccloss}, when choosing $\alpha \in [0,1]$, $\distillationlosscc{\alpha}$ can continuously interpolate between lower and upper bounds to the worst-case feature-space distillation loss in equation~\eqref{eq:ideal-distillation-loss}.
While $\S$\ref{sec:ccdist} advocates for the use of the same $\alpha$ coefficient employed by CC-IBP in order to closely mirror its loss (see lemma~\ref{prop:convexcombination}), we expect other choices of the $\alpha$ coefficient to work well on selected benchmarks.

\subsubsection{CC-Dist$_0$ and CC-Dist$_1$}

We here evaluate the behavior of two training schemes that correspond to other special cases of the CC-Dist distillation loss.
Specifically, we consider the use of $\distillationlosscc{0}=\distillationlosslb$ and  $\distillationlosscc{1}=\distillationlossub$ on top of CC-IBP, calling the resulting training schemes CC-Dist$_0$ and CC-Dist$_1$, respectively:
\begin{equation*}
	\begin{aligned}
		&\mathcal{L}^{\text{CC-Dist}_0} (\alpha, \beta; \xb, y) := \ccibp + \beta\ \distillationlosscc{0}, \\
		&\mathcal{L}^{\text{CC-Dist}_1} (\alpha, \beta; \xb, y) := \ccibp + \beta\ \distillationlosscc{1}.
	\end{aligned}
\end{equation*}
\looseness=-1

Keeping $\ell_1$ regularization and the $\alpha$ coefficient of the CC-IBP loss fixed, we first tested different distillation coefficients $\beta$ while using the teachers employed for the CC-Dist models (see tables~\ref{fig:teacher-student}~and~\ref{fig:hyperparameters}), finding that $\beta = \nicefrac{5}{w}$ yields the largest natural accuracy improvement for both methods on $\epsilon=\nicefrac{8}{255}$, and a strong trade-off between certified robustness and standard performance for CC-Dist$_0$ on $\epsilon=\nicefrac{2}{255}$. 
We found $\beta = \nicefrac{\left(5 \times 10^{-4}\right)}{w}$ to work better for CC-Dist$_1$ on $\epsilon=\nicefrac{2}{255}$.
Then, using the above distillation coefficients, and similarly to what done for CC-Dist (see appendix~\ref{sec:hyperparameters}), we tested the two methods using teachers with varying $\ell_1$ regularization and trained for $30$ epochs.
Aiming to compare them with CC-Dist, Table~\ref{fig:special-cases} reports results for the best-performing CC-Dist$_0$ and CC-Dist$_1$ models from the above procedure that display similar performance profiles with CC-Dist.

Table~\ref{fig:special-cases} shows that both CC-Dist$_0$ and CC-Dist$_1$ work well across the considered CIFAR-10 settings. 
As we would expect considering the low employed $\alpha$ coefficient (see table~\ref{fig:hyperparameters}), CC-Dist$_0$ strictly outperforms CC-Dist$_1$ on $\epsilon=\nicefrac{2}{255}$, where both methods nevertheless improve on both the standard performance and the certified accuracy of CC-IBP.
While CC-Dist$_0$ displays performance comparable to CC-Dist in this setting, CC-Dist$_1$ performs markedly worse.
On $\epsilon=\nicefrac{8}{255}$, where the employed $\alpha$ coefficient is instead relatively large (see table~\ref{fig:hyperparameters}), CC-Dist$_1$ strictly outperforms CC-Dist$_0$ on both reported performance profiles, indicating the benefit of distilling onto feature IBP bounds in this context.
Nevertheless, both methods can produce models improving on CC-IBP for both metrics on $\epsilon=\nicefrac{8}{255}$, with the corresponding CC-Dist$_0$ model strictly outperformed by CC-Dist.

\looseness=-1
We believe that the above results further confirm the effectiveness of using empirically-robust models to improve certified training via knowledge distillation. 
As the relative performance of CC-Dist$_0$ and CC-Dist$_1$ varies depending on the experimental setting, we advocate for the use of CC-Dist as described in $\S$\ref{sec:ccdist} due to its adaptive nature.
Nevertheless, given that most considered benchmarks require relatively low $\alpha$ coefficients (see table~\ref{fig:hyperparameters}), we expect CC-Dist$_0$ to be a valid alternative to CC-Dist on TinyImageNet and downscaled ImageNet.
We note, however, that re-running the above experiments with multiple seeds could lead to different conclusions: this does not affect the main contribution of this work, which is to show the benefits of distillation from adversarially-trained teachers.

\subsubsection{CC-Dist-Clean}

Recalling that $\distillationlosscc{0} = \left\lVert   h (\xb_{\text{adv}}) - \tlatent \right\rVert^2_2$, we point out that setting $\xb_{\text{adv}} = \xb$ would result in the standard feature-space distillation loss (see $\S$\ref{sec:dist-background}).
We hence consider this as an additional special case of $\distillationlosscc{\alpha}$, calling the resulting algorithm CC-Dist-Clean:
\begin{equation*}
	\mathcal{L}^{\text{CC-Dist-Clean}} (\alpha, \beta; \xb, y) := \ccibp + \beta\  \left\lVert   h (\xb) - \tlatent \right\rVert^2_2.
\end{equation*}

Keeping the teachers fixed to those employed in table~\ref{fig:ccdist-results} for CC-Dist, we tested the performance of CC-Dist-Clean for varying $\beta$ parameters on CIFAR-10, and report the best results in table~\ref{fig:special-cases}.
CC-Dist-Clean displays performance profiles similar to CC-Dist$_0$, suggesting that the use of adversarial features is not crucial when lower bounding the worst-case distillation loss via $\distillationlosslb$.
We advocate for  CC-Dist$_0$ over CC-Dist-Clean as the latter requires an extra clean forward pass, whereas the adversarial inputs come at no additional cost to CC-IBP.

\subsection{Logit-based distillation} \label{sec:logit-distillation}

\begin{table}[b!]
%	\vspace{-7pt}
	\sisetup{detect-all = true}
	
	\centering
	
	\scriptsize
	\setlength{\tabcolsep}{4pt}
	\aboverulesep = 0.3mm  % 0.605mm
	\belowrulesep = 0.1mm  % 0.984mm
	\caption{\small 
		\looseness=-1
		Comparison between the logit-space distillation from equation~\eqref{eq:cc-logit-dist} and the CC-Dist models from table~\ref{fig:ccdist-literature}.%
		\label{fig:logit-based}}
%	\vspace{-5pt}
	\begin{adjustbox}{width=.62\textwidth, center}
		
		\begin{tabular}{
				c
				c
				l
				S[table-format=2.2]
				S[table-format=2.2]
			} 
			
			\toprule \\[-5pt]
			
			\multicolumn{1}{c}{\multirow{1}{*}{Dataset}} &
			\multicolumn{1}{c}{\multirow{1}{*}{$\epsilon$}} &
			\multicolumn{1}{c}{\multirow{1}{*}{Method}} &
			\multicolumn{1}{c}{Standard acc. [\%]}  &
			\multicolumn{1}{c}{Certified acc. [\%]} \\
			
			\cmidrule(lr){1-2} \cmidrule(lr){3-3} \cmidrule(lr){4-5} \\[-5pt]
			
			\multirow{4.5}{*}{CIFAR-10} & \multirow{2}{*}{$\frac{2}{255}$} &
			
			\multicolumn{1}{ c }{\textsc{CC-Dist}}
			& 81.55 &  64.60 \\
			
			& & \multicolumn{1}{ c }{\textsc{Eq.}~\eqref{eq:cc-logit-dist}} 
			& 80.61  &  63.65 \\ 
			
			\cmidrule(lr){2-5} \\[-5pt]
			
			& \multirow{2}{*}{$\frac{8}{255}$} &
			
			\multicolumn{1}{ c }{\textsc{CC-Dist}} 
			& 55.13 &  35.52 \\
			
			& & \multicolumn{1}{ c }{\textsc{Eq.}~\eqref{eq:cc-logit-dist}} 
			& 54.76 &  35.20 \\
			
			\bottomrule 
			
		\end{tabular}
	\end{adjustbox}
\end{table}

\looseness=-1
We now compare the results of $\distillationlosscc{\alpha}$ with a loss that instead seeks to directly distill onto the CC-IBP convex combinations $\zdiffcc$. 
Specifically, it employs a KL term between $\zdiffcc$ and the teacher logit differences $\zb_{t}(\xb, y)$, resulting in the following training loss:
\begin{equation}
	\ccibp + T^2 \beta\ \text{KL}_{T} \left(-\zdiffcc, -\zb_{t}(\xb, y)\right).
	\label{eq:cc-logit-dist}
\end{equation}
\looseness=-1
Keeping all other hyper-parameters fixed (including the teacher model) to the values reported in table~\ref{fig:hyperparameters}, we tested different distillation coefficients $\beta$, using $T=20$. 
We found that, compared to CC-IBP, the best performance profile using equation~\eqref{eq:cc-logit-dist} were obtained using $\beta = \nicefrac{\left(10^{4}\right)}{w}$ for both considered CIFAR-10 setups, whose results are reported in table~\ref{fig:logit-based}.
CC-Dist outperforms the logit-space distillation loss in both settings.
We hypothesize this may be linked to the more informative content of the model features: allowing the student to learn a markedly different classification head from the teacher may be beneficial in the context of the employed loss functions.

\subsection{PreActResNet18 teachers}

\begin{table}[t!]
%	\vspace{-7pt}
	\sisetup{detect-all = true}
	
	\centering
	
	\scriptsize
	\setlength{\tabcolsep}{4pt}
	\aboverulesep = 0.3mm  % 0.605mm
	\belowrulesep = 0.1mm  % 0.984mm
	\caption{\small 
		\looseness=-1
		Evaluation of the effect of PreActResNet18 (\texttt{PRN18}) teachers (tch.) on the TinyImageNet performance of CC-Dist, compared to teachers with the same architecture as the student.
		\label{fig:prn18}}
	\vspace{5pt}
	\begin{adjustbox}{width=.96\textwidth, center}
		
		\begin{tabular}{
				c
				c
				l
				S[table-format=2.2]
				S[table-format=2.2]
				S[table-format=2.2]
				S[table-format=2.2]
			} 
			
			\toprule \\[-5pt]
			
			\multicolumn{1}{c}{Dataset} &
			\multicolumn{1}{c}{$\epsilon$} &
			\multicolumn{1}{c}{Tch. arch.} &
			\multicolumn{1}{c}{Std. acc. [\%]}  &
			\multicolumn{1}{c}{Cert. acc. [\%]} &
			\multicolumn{1}{c}{Tch. std. acc. [\%]}  &
			\multicolumn{1}{c}{Tch. PGD-40 acc. [\%]} \\
			
			\cmidrule(lr){1-2} \cmidrule(lr){3-3} \cmidrule(lr){4-5} \cmidrule(lr){6-7} \\[-5pt]
			
			\multirow{2}{*}{TinyImageNet} & \multirow{2}{*}{$\frac{1}{255}$} &
			
			\multicolumn{1}{ c }{\texttt{CNN-7}} 
			& 44.08 &  27.40 & 47.18 & 36.07 \\
			
			& & \multicolumn{1}{ c }{\texttt{PRN18}}  & 43.03 & 26.09
			& 50.90 & 39.98 \\ 
			
			\bottomrule 
			
		\end{tabular}
	\end{adjustbox}
\end{table}

We now present a preliminary evaluation of the effect of employing a PreActResNet18 (\texttt{PRN18}) architecture~\citep{He2016} as teacher for CC-Dist. 
In order to ensure that the feature space of the teacher is set to the output of a ReLU, in compliance with the student model, we place the PreActResNet18 average pooling layer before the last ReLU and batch normalization (the standard PreActResNet18 architecture places it instead before the last linear layer).
Focusing on TinyImageNet, and keeping $\alpha=3 \times 10^{-3}$ and the $\ell_1$ coefficient to $5 \times 10^{-5}$ as for the CC-IBP model on this setup (see table~\ref{fig:hyperparameters}), we tested various \texttt{PRN18} teachers trained for $30$ epochs with varying $\ell_1$ regularization as CC-Dist teachers.
Table~\ref{fig:prn18} compares the performance of the ensuing CC-Dist model having the largest natural accuracy with the best model obtained through same-architecture teachers for these hyper-parameters, and the respective teachers.
CC-Dist draws no benefit from the \texttt{PRN18} teacher (which was trained with $\ell_1$ coefficient equal to $5 \times 10^{-5}$), in spite of the fact that it displays stronger performance than the \texttt{CNN-7} teacher.
While we are hopeful that better CC-Dist results could be obtained by training \texttt{PRN18} teachers for longer (the considered teachers are trained for $30$ epochs, as opposed to the $100$ epochs employed for the \texttt{CNN-7} teacher), we leave this for future work.

\subsection{Runtime measurements}

\looseness=-1
In order to assess the training overhead associated with our distillation scheme, we here present runtime measurements and estimates for CC-Dist and CC-IBP.
Specifically, we provide the training runtime of both methods, separately including also the teacher training runtime for CC-Dist, and estimates of the verification runtimes for the trained models from table~\ref{fig:ccdist-results}.
These experiments were carried out using an Nvidia RTX 8000 GPU, and 6 cores of an AMD EPIC 7302 CPU.
Table~\ref{fig:runtime} shows that, under the training schedules of appendix~\ref{sec:training-schedules}, CC-Dist is associated with minimal training overhead on the considered CIFAR-10 settings.
This overhead increases to respectively almost $70\%$ and $40\%$ for TinyImageNet and downscaled ImageNet, where stronger teachers (with relatively longer training schedules) are required in order to maximize performance.
Finally, we note that the increased certified accuracy of the CC-Dist models (see table~\ref{fig:ccdist-results}) comes with an increased verification runtime.
%, as expected from the increases IBP loss resulting from the distillation process (see table~\ref{fig:distillation-effect}).

\begin{table}[h!]
%	\vspace{-7pt}
	\sisetup{
		detect-all = true,
		exponent-mode = scientific,
		round-mode = figures,
		round-precision = 4
	}
	
	\centering
	
	\scriptsize
	\setlength{\tabcolsep}{4pt}
	\aboverulesep = 0.3mm  % 0.605mm
	\belowrulesep = 0.1mm  % 0.984mm
	\caption{\small 
		\looseness=-1
		Training runtime measurements for CC-Dist, their respective teachers, and CC-IBP under the training schedules of appendix~\ref{sec:training-schedules}.%
		\label{fig:runtime}}
	\vspace{5pt}
	\begin{adjustbox}{width=.95\textwidth, center}
		
		\begin{tabular}{
				c
				c
				l
				S[table-format=1.3e1]
				S[table-format=1.3e1]
				S[table-format=1.3e1]
			} 
			
			\toprule \\[-5pt]
			
			\multicolumn{1}{c}{\multirow{1}{*}{Dataset}} &
			\multicolumn{1}{c}{\multirow{1}{*}{$\epsilon$}} &
			\multicolumn{1}{c}{\multirow{1}{*}{Method}} &
			\multicolumn{1}{c}{Training runtime [s]} &
			\multicolumn{1}{c}{Teacher training runtime [s]} &
			\multicolumn{1}{c}{Estimated$^\dagger$ verification runtime [s]}
			\\
			
			\cmidrule(lr){1-2} \cmidrule(lr){3-3} \cmidrule(lr){4-6} \\[-5pt]
			
			\multirow{4.5}{*}{CIFAR-10} & \multirow{2}{*}{$\frac{2}{255}$} &
			
			\multicolumn{1}{ c }{\textsc{CC-Dist}}
			& 16045.41909 &  2224.7256 &  109716.254 \\
			
			& & \multicolumn{1}{ c }{\textsc{CC-IBP}} 
			& 15678.20053  &  / & 70889.845 \\ 
			
			\cmidrule(lr){2-6} \\[-5pt]
			
			& \multirow{2}{*}{$\frac{8}{255}$} &
			
			\multicolumn{1}{ c }{\textsc{CC-Dist}} 
			& 16531.98289  &  2223.14948 & 10900.3978 \\
			
			& & \multicolumn{1}{ c }{\textsc{CC-IBP}} 
			& 16002.23646 &  / & 6238.3524 \\
			
			\cmidrule(lr){1-6} \\[-5pt]
			
			\multirow{2}{*}{TinyImageNet} & \multirow{2}{*}{$\frac{1}{255}$} &
			
			\multicolumn{1}{ c }{\textsc{CC-Dist}}
			& 54316.04395  &  34702.53902  & 277752.9068 \\  
			
			& & \multicolumn{1}{ c }{\textsc{CC-IBP}} 
			& 52462.51405  &  / & 160825.3278 \\ 
			
			\cmidrule(lr){1-6} \\[-5pt]
			
			\multirow{2}{*}{ImageNet64} & \multirow{2}{*}{$\frac{1}{255}$} &
			
			\multicolumn{1}{ c }{\textsc{CC-Dist}}
			& 322843.10673  &  108904.12346 & 1210079.687 \\
			
			& & \multicolumn{1}{ c }{\textsc{CC-IBP}} 
			&  309470.16224 &  / &  779384.0 \\ 
			
			\bottomrule 
			
			\\[-5pt]
			\multicolumn{6}{ l }{
				\shortstack[l]{$^\dagger$Extrapolated from measurements over the first $500$ test images.} 
			}
			
		\end{tabular}
	\end{adjustbox}
\end{table}

\begin{table}[b!]
%	\vspace{-7pt}
	\sisetup{detect-all = true}
			
	\centering
			
	\scriptsize
	\setlength{\tabcolsep}{4pt}
	\aboverulesep = 0.3mm  % 0.605mm
	\belowrulesep = 0.1mm  % 0.984mm
	\caption{\small 
		\looseness=-1
		Effect of distillation beyond \texttt{CNN-7} students on CIFAR-10 with $\epsilon=\nicefrac{2}{255}$ for CC-Dist. \label{fig:diff-architectures}}
%	\vspace{-5pt}
		\begin{subtable}{.49\textwidth}
			\captionsetup{justification=centering, labelsep=space, margin=6pt,labelformat=parens}
			\caption{\small 
				\looseness=-1
				Modified \texttt{CNN-7} with tanh activation functions. \label{fig:tanh}}
%			\vspace{5pt}
		\begin{adjustbox}{width=.93\textwidth, center}
				\begin{tabular}{
						l
						S[table-format=2.2]
						S[table-format=2.2]
					} 
					\toprule \\[-5pt]
					\multicolumn{1}{c}{Method} &
					\multicolumn{1}{c}{Std. acc. [\%]}  &
					\multicolumn{1}{c}{CROWN acc. [\%]}  \\
					\cmidrule(lr){1-1} \cmidrule(lr){2-3} \\[-5pt]
					\multicolumn{1}{ c }{\texttt{CC-Dist}} 
					& 72.38 & 49.42 \\
					\multicolumn{1}{ c }{\texttt{CC-IBP}}  
						& 71.29 & 47.39  \\ 
					\bottomrule 
				\end{tabular}
		\end{adjustbox}
		\end{subtable}
		\begin{subtable}{.49\textwidth}
			\captionsetup{justification=centering, labelsep=space, margin=6pt,labelformat=parens}
			\caption{\small 
				\looseness=-1
				PreActResNet18 architecture. \label{fig:prn18-both}}
%			\vspace{5pt}
			\begin{adjustbox}{width=.93\textwidth, center}
				
					\begin{tabular}{
							l
							S[table-format=2.2]
							S[table-format=2.2]
						} 
						\toprule \\[-5pt]
						\multicolumn{1}{c}{Method} &
						\multicolumn{1}{c}{Std. acc. [\%]}  &
						\multicolumn{1}{c}{CROWN acc. [\%]}  \\
						\cmidrule(lr){1-1} \cmidrule(lr){2-3} \\[-5pt]
						\multicolumn{1}{ c }{\texttt{CC-Dist}} 
						& 74.01 & 53.63 \\
						\multicolumn{1}{ c }{\texttt{CC-IBP}}  
						& 73.40 & 52.88  \\ 
						\bottomrule 
					\end{tabular}
			\end{adjustbox}
		\end{subtable}
\end{table}

\subsection{Different architecture and activation function} \label{sec:app-diff-architectures}
The main results focus on the \texttt{CNN-7} architecture owing to its state-of-the-art performance and prevalence in the relevant literature~\citep{DePalma2024,Mao2023,Mueller2023,Shi2021}.
We here investigate whether the distillation process is beneficial beyond this context by studying the relative performance of CC-Dist on a different activation function and on a different architecture for both the teacher and the model, focussing on CIFAR-10 with $\epsilon=\nicefrac{2}{255}$.
For the activation experiment, we modify \texttt{CNN-7} to employ the hyperbolic tangent (tanh), testing different teachers trained for $100$ epochs (choosing a teacher $\ell_1$ coefficient of $2\times10^{-5}$), and varying $\beta$ values (settling for $\beta=\nicefrac{5}{w}$ as for the experiments in table~\ref{fig:ccdist-results}).
For the architecture experiment, we use PreActResNet18 (\texttt{PRN18}) for both the teacher and the student, testing different teachers trained for $30$ epochs (settling on a teacher with $\ell_1$ coefficient of $5\times10^{-6}$), and using $\beta=\nicefrac{5}{w}$.
In both cases, owing to the lack of support for either model from the OVAL branch-and-bound framework~\citep{Bunel2018,BunelDP20,improvedbab} employed throughout the paper, we use CROWN~\citep{Zhang2018} from~\texttt{auto\_LiRPA}~\citep{Xu2020} as post-training verification algorithm.
Table~\ref{fig:diff-architectures} shows that CC-Dist successfully improves robustness-accuracy trade-offs for both the considered settings, demonstrating the wider applicability of the proposed distillation technique. 
\subsection{Distillation on earlier latents} \label{sec:app-earlier-latents}
Throughout the paper, we define the feature space as the activations before last affine network layer, in accordance with the conditions of lemma~\ref{prop:ccloss}. We here investigate the effect of computing on the distillation loss $\distillationlosscc{\alpha}$ on an earlier feature space, corresponding to the activations before the penultimate affine layer of \texttt{CNN-7}, focusing on CC-Dist.
In particular, we focus on CIFAR-10 with $\epsilon=\nicefrac{2}{255}$, keeping the teacher model fixed to the one used for tables~\ref{fig:ccdist-results}~and~\ref{fig:teacher-student} and varying the $\beta$ coefficient to account for the change in the feature space: we found $\beta=\nicefrac{2}{w}$ to yield the maximize performance in this context.
Table~\ref{fig:earlier-latent-distillation} shows that distillation for the last affine layer yields strictly better robustness-accuracy trade-offs than those presented in table~\ref{sec:main-exp}. 
As distilling on earlier layers implies that a larger portion of the network is exclusively trained using the CC-IBP loss, we ascribe this to insufficient teacher-student coupling.
\subsection{Clean teachers} \label{sec:app-clean-teachers}
In spite of our focus on learning from adversarially-trained teachers, $\distillationlosscc{\alpha}$ only makes use of the clean features of the teacher model.
In order to investigate whether the empirical robustness of the teacher is indeed necessary for an effective distillation process, we here study the effect of employing standard-trained teachers for CC-Dist.
In particular, focusing on CIFAR-10 with $\epsilon=\nicefrac{2}{255}$, we keep $\beta=\nicefrac{2}{w}$ and test the use of SGD-trained teachers trained with varying $\ell_1$ regularization coefficients ($5\times10^{-6}$ being the chosen $\ell_1$ coefficient).
Table~\ref{fig:clean-teachers} shows that the use of standard teachers markedly worsens the certified accuracy of the student, which is now inferior to the one associated to CC-IBP (see table~\ref{fig:ccdist-literature}).
These results demonstrate that a robust teacher representation is a key requirement for the success of the distillation process
\subsection{Distilling from adversarial teacher features} \label{sec:app-t-adv-logits}
\looseness=-1
As explained in the previous appendix, all the distillation losses from $\S$\ref{sec:method} perform distillation from the clean teacher features $\tlatent$. 
Having shown that robustness of the teacher model is crucial to maximize performance, we investigate how the input point on which the teacher features are computed impacts the effectiveness of the distillation process.
Specifically, focusing on CC-Dist and CIFAR-10, we consider the effect of using teacher features computed on $\xb_{\text{adv}}$, the adversarial input computed on the student model, referring to $h^t (\xb_{\text{adv}})$ as adversarial teacher features.
Table~\ref{fig:teacher-adversarial-features} shows that clean teacher features yield strictly better performance for $\epsilon=\nicefrac{2}{255}$ and are hence to be preferred.
We hypothesize that, in spite of the teacher's robustness, its clean features may provide a higher-quality representation to learn from.

\begin{table*}[t!]
	%	\vspace{-7pt}
	\sisetup{detect-all = true}
	
	\centering
	
	\scriptsize
	\setlength{\tabcolsep}{4pt}
	\aboverulesep = 0.3mm  % 0.605mm
	\belowrulesep = 0.1mm  % 0.984mm
	\caption{\small 
		Effect of distillation from the adversarial teacher features on CC-Dist.
		\label{fig:teacher-adversarial-features}}
	\begin{adjustbox}{width=.79\textwidth, center}
		
		\begin{tabular}{
				c
				c
				l
				S[table-format=2.2]
				S[table-format=2.2]
				S[table-format=2.2]
			} 
			
			\toprule \\[-5pt]
			
			\multicolumn{1}{c}{\multirow{1}{*}{Dataset}} &
			\multicolumn{1}{c}{\multirow{1}{*}{$\epsilon$}} &
			\multicolumn{1}{c}{\multirow{1}{*}{Teacher features}} &
			\multicolumn{1}{c}{Standard acc. [\%]}  &
			\multicolumn{1}{c}{AA acc. [\%]} &
			\multicolumn{1}{c}{Certified acc.$^\dagger$ [\%]} \\
			
			\cmidrule(lr){1-2} \cmidrule(lr){3-3} \cmidrule(lr){4-6} \\[-5pt]
			
			\multirow{5}{*}{CIFAR-10} & \multirow{2}{*}{$\frac{2}{255}$} &
			
			\multicolumn{1}{ c }{\textsc{Clean}}
			& 81.55 & 70.71 &  64.60 \\
			
			& & \multicolumn{1}{ c }{\textsc{Adversarial}} 
			& 81.39  & 70.16 &  64.15 \\ 
			
			\cmidrule(lr){2-6} \\[-5pt]
			
			& \multirow{2}{*}{$\frac{8}{255}$} &
			
			\multicolumn{1}{ c }{\textsc{Clean}} 
			& 55.13  & 36.80 &  35.52 \\
			
			& & \multicolumn{1}{ c }{\textsc{Adversarial}} 
			& 55.64	& 36.84	& 35.04 \\
			
			\bottomrule 
			
		\end{tabular}
	\end{adjustbox}
	%	\vspace{-3pt}
	%	\vspace{-5pt}
\end{table*}

\section{Significance of the empirical results} \label{sec:significance}

In order to put the improvements from our distillation framework into context, table~\ref{fig:literature-improvements} reports those associated with relevant recent work in $\ell_\infty$ certified training for ReLU networks.
We exclude CIFAR-10 with $\epsilon=\nicefrac{8}{255}$, where Lipschitz-based methods have been the unquestioned state-of-the-art~\citep{Zhang2022rethinking}, and consider both improvements to previous literature results, and to re-tuned baselines within each work (for instance, results reported in~\citet{DePalma2024} are compared against their own SABR results~\citep[table 9]{DePalma2024}).
While, as shown in table~\ref{fig:runtime}, distillation improvements are associated with a non-negligible additional computational overhead at training or verification time, they are in line with the best algorithmic improvements in the area, pointing to their clear significance within the context of certified training.

\begin{table*}[h!]
	%	\vspace{-7pt}		
	\sisetup{detect-all = true}
	
	\setlength{\heavyrulewidth}{0.09em}
	\setlength{\lightrulewidth}{0.5em}
	\setlength{\cmidrulewidth}{0.03em}
	
	\centering
	
	\scriptsize
	\setlength{\tabcolsep}{4pt}
	\aboverulesep = 0.3mm  % 0.605mm
	\belowrulesep = 0.1mm  % 0.984mm
	\caption{%\small 
		\looseness=-1
		Comparison of CC-Dist average improvements (from table~\ref{fig:ccdist-results}) with those obtained by relevant previous work in the area on CIFAR-10 with $\epsilon=\nicefrac{2}{255}$, TinyImageNet and downscaled ImageNet. These are computed as the difference between the single-seed results from each of these works with the best-performing model from earlier methods.
		We consider both previously-reported literature results and re-tuned baselines within each of these works. 
		\label{fig:literature-improvements}}
	\begin{adjustbox}{width=.93\textwidth, center}
		\begin{tabular}{
				c
				c
				l
				c
				S[table-format=1.2]
				S[table-format=1.2]
			} 
			
			\toprule \\[-5pt]
			
			\multicolumn{1}{ c }{Dataset} &
			\multicolumn{1}{ c }{$\epsilon$} &
			\multicolumn{1}{ c }{Method} &
			\multicolumn{1}{ c }{Source} &
			\multicolumn{1}{ c }{Standard acc. impr. [\%]} &
			\multicolumn{1}{ c }{Certified acc. impr. [\%]}  \\
			
			\cmidrule(lr){1-2} \cmidrule(lr){2-2} \cmidrule(lr){3-4} \cmidrule(lr){5-6} \\[-5pt]
			
			\multirow{4}{*}{CIFAR-10} & \multirow{4}{*}{$\frac{2}{255}$} &
			
			\multicolumn{1}{ c }{\textsc{CC-Dist}} & this work
			& 1.76 & 0.92 \\
			
			\cmidrule(lr){3-6} \\[-6pt]
			
			& & \multicolumn{1}{ c }{\textsc{CC-IBP}} & \citet{DePalma2024}
			& 0.20 & 0.50\\
			
			& & \multicolumn{1}{ c }{\textsc{STAPS}} & \citet{Mao2023}
			& 0.52 & 0.14  \\
			
			& & \multicolumn{1}{ c }{\textsc{SABR}} & \citet{Mueller2023}
			& 1.05 & 0.87  \\
						
			\cmidrule{1-6} \\[-6pt]
			
			\multirow{4}{*}{TinyImageNet} & \multirow{4}{*}{$\frac{1}{255}$} &
			
			\multicolumn{1}{ c }{\textsc{CC-Dist}} & this work
			& 2.55  & 1.16 \\
			
			\cmidrule(lr){3-6} \\[-6pt]
			
			& & \multicolumn{1}{ c }{\textsc{Exp-IBP}} & \citet{DePalma2024}
			& 0.03 & 0.33 \\
			
			& & \multicolumn{1}{ c }{\textsc{STAPS}} & \citet{Mao2023}
			& 0.13 & 1.70  \\
			
			& & \multicolumn{1}{ c }{\textsc{SABR}} & \citet{Mueller2023}
			& 2.93 & 2.59   \\
			
			\cmidrule{1-6} \\[-6pt]
			
			\multirow{3}{*}{ImageNet64} & \multirow{3}{*}{$\frac{1}{255}$} &
			
			\multicolumn{1}{ c }{\textsc{CC-Dist}} & this work
			& 2.21 &  0.56  \\ 
			
			\cmidrule(lr){3-6} \\[-6pt]
			
			& & \multicolumn{1}{ c }{\textsc{Exp-IBP}} & \citet{DePalma2024}
			& 2.40 & 0.91  \\
			
			\bottomrule 
		\end{tabular}
	\end{adjustbox}
	%	\vspace{-7pt}
\end{table*}

\section{Details on the employed notation} \label{sec:app-notation}
We here provide additional details on the employed notation.
\paragraph{Student and teacher models}
Input-to-logits maps are defined as the composition of a feature map, denoted using the letter $h$, and classification heads, denoted using the letter $g$.
The (sub-)networks for the teacher model are denoted by the use of the $t$ superscript (e.g, $h^t$).
\paragraph{Bounds to worst-case quantities}
The worst-case classification loss, equation~\eqref{eq:worst-case}, and the worst-case distillation loss, equation~\eqref{eq:ideal-distillation-loss}, are superscripted by the relative local perturbation set $\mathcal{C}_\epsilon$ around the input $\xb$.
Lower and upper bounds to both quantities are denoted through lower and upper bars, respectively (e.g., $\advloss$), with the superscript preserved to stress their local validity.
A similar notation is employed for the local bounds to the worst-case logit differences (e.g, $\zdifflb$, which lower bounds equation~\ref{eq:verification}.
Convex combinations between lower and upper bounds, parametrized by $\alpha$, are denoted by subscript ``$\text{cc}$'' standing for convex combination (e.g., $\distillationlosscc{\alpha}$), without any bars.
For all these worst-case quantities, lower bounds corresponding to evaluations on concrete inputs from adversarial attacks, and the upper bounds are obtained through network convex relaxations (specifically, IBP).
\paragraph{Bounds to the student features}
We again employ a similar notation when bounding the student features $h (\xb) $, defined in proposition~\ref{prop:ubloss}.
In this context, both the lower and the upper bounds are obtained through IBP.
Convex combinations between the adversarial student latents $h (\xb_{\text{adv}})$ and the IBP lower and upper bounds, as defined in equation~\ref{eq:latentdiffcc}, are denoted by subscript ``$\text{cc}$'' and lower and upper bars, respectively (e.g., $\latentdiffccub$).

\begin{table}[b!]
%	\vspace{-7pt}
	\sisetup{detect-all = true}
	
	\centering
	
	\scriptsize
	\setlength{\tabcolsep}{4pt}
	\aboverulesep = 0.3mm  % 0.605mm
	\belowrulesep = 0.1mm  % 0.984mm
	\caption{\small 
		\looseness=-1
		Effect of computing the distillation loss $\distillationlosscc{\alpha}$ on an earlier feature space for CC-Dist.
		\label{fig:earlier-latent-distillation}}
%	\vspace{-5pt}
	\begin{adjustbox}{width=.6\textwidth, center}
			\begin{tabular}{
					c
					c
					l
					S[table-format=2.2]
					S[table-format=2.2]
				} 
				\toprule \\[-5pt]
				\multicolumn{1}{c}{Dataset} &
				\multicolumn{1}{c}{$\epsilon$} &
				\multicolumn{1}{c}{Feature space} &
				\multicolumn{1}{c}{Std. acc. [\%]}  &
				\multicolumn{1}{c}{Cert. acc. [\%]}  \\
				\cmidrule(lr){1-2} \cmidrule(lr){3-3} \cmidrule(lr){4-5} \\[-5pt]
				\multirow{2}{*}{CIFAR-10} & \multirow{2}{*}{$\frac{2}{255}$} &
				\multicolumn{1}{ c }{\texttt{Last}} 
				& 81.55 & 64.60 \\
				& & \multicolumn{1}{ c }{\texttt{Penultimate}}  
				& 80.84 & 63.11  \\ 
				\bottomrule 
			\end{tabular}
	\end{adjustbox}
\end{table}

\begin{table}[b!]
%	\vspace{-7pt}
	\sisetup{detect-all = true}
	
	\centering
	
	\scriptsize
	\setlength{\tabcolsep}{4pt}
	\aboverulesep = 0.3mm  % 0.605mm
	\belowrulesep = 0.1mm  % 0.984mm
	\caption{\small 
		\looseness=-1
		Effect of performing distillation from clean teachers for CC-Dist.
		\label{fig:clean-teachers}}
%	\vspace{-5pt}
	\begin{adjustbox}{width=.6\textwidth, center}
			\begin{tabular}{
					c
					c
					l
					S[table-format=2.2]
					S[table-format=2.2]
				} 
				\toprule \\[-5pt]
				\multicolumn{1}{c}{Dataset} &
				\multicolumn{1}{c}{$\epsilon$} &
				\multicolumn{1}{c}{Teacher training} &
				\multicolumn{1}{c}{Std. acc. [\%]}  &
				\multicolumn{1}{c}{Cert. acc. [\%]}  \\
				\cmidrule(lr){1-2} \cmidrule(lr){3-3} \cmidrule(lr){4-5} \\[-5pt]
				\multirow{2}{*}{CIFAR-10} & \multirow{2}{*}{$\frac{2}{255}$} &
				\multicolumn{1}{ c }{\texttt{PGD-10}} 
				& 81.55 & 64.60 \\
				& & \multicolumn{1}{ c }{\texttt{Standard}}  
				& 81.48 & 62.52  \\ 
				\bottomrule 
			\end{tabular}
	\end{adjustbox}
\end{table}

\clearpage

\newpage
\section*{NeurIPS Paper Checklist}

\begin{enumerate}

\item {\bf Claims}
    \item[] Question: Do the main claims made in the abstract and introduction accurately reflect the paper's contributions and scope?
    \item[] Answer: \answerYes{}
    \item[] Justification: Our claims match the presented techniques and the provided experimental results, which hold on popular benchmarks from the certifiably robust vision literature, which are the target of our work.

\item {\bf Limitations}
    \item[] Question: Does the paper discuss the limitations of the work performed by the authors?
    \item[] Answer: \answerYes{}
    \item[] Justification: Limitations are discussed in appendix~\ref{sec:app-limitations-broader}.

\item {\bf Theory assumptions and proofs}
    \item[] Question: For each theoretical result, does the paper provide the full set of assumptions and a complete (and correct) proof?
    \item[] Answer: \answerYes{}
    \item[] Justification: Assumptions are stated in the result statements. Proofs of technical results are provided in appendix~\ref{sec:app-proofs}.

    \item {\bf Experimental result reproducibility}
    \item[] Question: Does the paper fully disclose all the information needed to reproduce the main experimental results of the paper to the extent that it affects the main claims and/or conclusions of the paper (regardless of whether the code and data are provided or not)?
    \item[] Answer: \answerYes{}
    \item[] Justification: Pseudo-code is provided in appendix~\ref{sec:pseudo-code}. Information to reproduce the experiments can be found in appendix~\ref{sec:app-experiments-details}. 

\item {\bf Open access to data and code}
    \item[] Question: Does the paper provide open access to the data and code, with sufficient instructions to faithfully reproduce the main experimental results, as described in supplemental material?
    \item[] Answer: \answerYes{}
    \item[] Justification: Code for CC-Dist, SABR-Dist, MTL-Dist and IBP-R-Dist, along with sample CC-Dist student and teacher models, is provided at \url{https://github.com/alessandrodepalma/feature-distillation-for-certified-training}. The employed datasets are freely available online.

\item {\bf Experimental setting/details}
    \item[] Question: Does the paper specify all the training and test details (e.g., data splits, hyperparameters, how they were chosen, type of optimizer) necessary to understand the results?
    \item[] Answer: \answerYes{}
    \item[] Justification: Experimental settings are outlined in $\S$\ref{sec:experiments}; details, including hyper-parameters, are provided in appendix~\ref{sec:app-experiments-details}. 

\item {\bf Experiment statistical significance}
    \item[] Question: Does the paper report error bars suitably and correctly defined or other appropriate information about the statistical significance of the experiments?
    \item[] Answer: \answerYes{} 
    \item[] Justification: Our main CC-Dist experiments (table~\ref{fig:ccdist-results}) report means over six repetitions and their $95\%$ Confidence Intervals (CIs). 
    Our main SABR-Dist, MTL-Dist and IBP-R-Dist results (table~\ref{fig:sabr-distillation}) report means over four repetitions and $95\%$ CIs. The other experiments are run using a single seed, owing to the large cost associated with verification (a timeout of $600$ seconds per evaluation image is employed), and in line with related previous works~\citep{DePalma2024,Mao2023,Mueller2023}.

\item {\bf Experiments compute resources}
    \item[] Question: For each experiment, does the paper provide sufficient information on the computer resources (type of compute workers, memory, time of execution) needed to reproduce the experiments?
    \item[] Answer: \answerYes{}
    \item[] Justification: Information on the employed compute resources is provided in appendix~\ref{sec:app-experiments-details}.
    
\item {\bf Code of ethics}
    \item[] Question: Does the research conducted in the paper conform, in every respect, with the NeurIPS Code of Ethics \url{https://neurips.cc/public/EthicsGuidelines}?
    \item[] Answer: \answerYes{}
    \item[] Justification: The authors have read and adhered to NeurIPS Code of Ethics.

\item {\bf Broader impacts}
    \item[] Question: Does the paper discuss both potential positive societal impacts and negative societal impacts of the work performed?
    \item[] Answer: \answerYes{}
    \item[] Justification: See appendix~\ref{sec:app-limitations-broader}.
    
\item {\bf Safeguards}
    \item[] Question: Does the paper describe safeguards that have been put in place for responsible release of data or models that have a high risk for misuse (e.g., pre-trained language models, image generators, or scraped datasets)?
    \item[] Answer: \answerNA{}
    \item[] Justification: Owing to the relatively small scale of the models and to the foundational nature of the research, we believe that certifiably-robust models pose no such risks.

\item {\bf Licenses for existing assets}
    \item[] Question: Are the creators or original owners of assets (e.g., code, data, models), used in the paper, properly credited and are the license and terms of use explicitly mentioned and properly respected?
    \item[] Answer: \answerYes{}
    \item[] Justification: the sources of data and models are cited throughout the paper. Software acknowledgments are provided in appendix~\ref{sec:app-experiments-details}.

\item {\bf New assets}
    \item[] Question: Are new assets introduced in the paper well documented and is the documentation provided alongside the assets?
    \item[] Answer: \answerYes{}
    \item[] Justification: a \texttt{README} file is provided in the code repository.

\item {\bf Crowdsourcing and research with human subjects}
    \item[] Question: For crowdsourcing experiments and research with human subjects, does the paper include the full text of instructions given to participants and screenshots, if applicable, as well as details about compensation (if any)? 
    \item[] Answer: \answerNA{}
    \item[] Justification: the paper does not involve neither crowdsourcing nor research with human subjects.

\item {\bf Institutional review board (IRB) approvals or equivalent for research with human subjects}
    \item[] Question: Does the paper describe potential risks incurred by study participants, whether such risks were disclosed to the subjects, and whether Institutional Review Board (IRB) approvals (or an equivalent approval/review based on the requirements of your country or institution) were obtained?
    \item[] Answer: \answerNA{}
    \item[] Justification: the paper does not involve neither crowdsourcing nor research with human subjects.

\item {\bf Declaration of LLM usage}
    \item[] Question: Does the paper describe the usage of LLMs if it is an important, original, or non-standard component of the core methods in this research? Note that if the LLM is used only for writing, editing, or formatting purposes and does \emph{not} impact the core methodology, scientific rigor, or originality of the research, declaration is not required.
    %this research? 
    \item[] Answer: \answerNA{}
    \item[] Justification: No LLM impact on core methodology.

\end{enumerate}

\end{document}